\documentclass[11pt, letterpaper]{article}
\pdfoutput=1
\usepackage[authoryear]{natbib}
\usepackage{amsmath,amssymb,amsfonts,amsthm,bbm}
\usepackage{epic,eepic,epsfig,longtable}
\usepackage{multirow,verbatim}
\usepackage{array}
\usepackage{graphicx}
\usepackage{subcaption}
\usepackage{algpseudocode}
\usepackage{ctable} 

\usepackage[ruled]{algorithm2e}

\usepackage{epsfig}
\usepackage{setspace}
\usepackage{url}
\usepackage{color}

\newcommand{\ASGD}{ASGD\xspace}

\newcommand{\cifar}{\emph{CIFAR-10}\xspace}

\newcommand{\Hier}{\emph{Hier-AVG}\xspace}
\newcommand{\KAVG}{\emph{K-AVG}\xspace}
\newcommand{\resnet}{\emph{ResNet-18}\xspace}
\newcommand{\imagenet}{\emph{ImageNet-1K}\xspace}
\newcommand{\mobilenet}{\emph{MobileNet}\xspace}

\newcommand{\vgg}{\emph{VGG19}\xspace}
\newcommand{\googlenet}{\emph{GoogLeNet}\xspace}

\makeatletter
\def\singlespace{\def\baselinestretch{1}\@normalsize}

\makeatletter
\def\singlespace{\def\baselinestretch{1}\@normalsize}

\numberwithin{equation}{section}

\newcommand{\bfm}[1]{\ensuremath{\mathbf{#1}}}

     \def\EE{\mathbb{E}}

     \def\NN{\mathbb{N}}

     \def\RR{\mathbb{R}}

\def\bw{\bfm w}

\def\today{\ifcase\month\or
  January\or February\or March\or April\or May\or June\or
  July\or August\or September\or October\or November\or December\fi
  \space\number\day, \number\year}

\newdimen\biblioindent    \biblioindent=30pt

 at 8truept

\newcommand{\beq}{\begin{equation}}
  \newcommand{\eeq}{\end{equation}}
\newcommand{\beqn}{\begin{eqnarray}}
  \newcommand{\eeqn}{\end{eqnarray}}
\newcommand{\beqnn}{\begin{eqnarray*}}
  \newcommand{\eeqnn}{\end{eqnarray*}}

\allowdisplaybreaks
\usepackage{verbatim}
\renewcommand{\baselinestretch}{1.66}
\numberwithin{equation}{section}
\theoremstyle{plain}
\newtheorem{theorem}{Theorem}[section]

\theoremstyle{definition}

\theoremstyle{definition}

\theoremstyle{definition}
\newtheorem{lemma}{Lemma}

\theoremstyle{definition}
\newtheorem{assumption}{Assumption}

\theoremstyle{definition}

\theoremstyle{definition}

\usepackage{mathtools}

\begin{document}
\title{A Distributed Hierarchical Averaging SGD Algorithm: Trading Local Reductions for Global Reductions}
\author{
Fan Zhou$^1$\thanks{Supported by NSF Grants DMS-1509739 and CCF-1523768}, 
Guojing Cong$^2$
\\ 
$^1$ School of Mathematics, Georgia Institute of Technology \\
$^2$ IBM Thomas J. Watson Research Center\\
fzhou40@math.gatech.edu,
gcong@us.ibm.com
}

\date{\today}
\maketitle

\begin{abstract}
Reducing communication in training large-scale machine learning applications on distributed platform is still a big challenge.
To address this issue, we propose a distributed hierarchical averaging stochastic gradient descent (\Hier) algorithm with infrequent global reduction by introducing local reduction.
As a general type of parallel SGD, \Hier can reproduce several popular synchronous parallel SGD variants by adjusting its parameters.
We show that \Hier with infrequent global reduction can still achieve standard convergence rate for non-convex optimization problems.
In addition, we show that more frequent local averaging with more participants involved can lead to faster training convergence. 
By comparing \Hier with another popular distributed training algorithm \KAVG, we show that through deploying local averaging with fewer number of global averaging, \Hier can still achieve comparable training speed while frequently get better test accuracy. This indicates that local averaging can serve as an alternative remedy to effectively reduce communication overhead when the number of learners is large. 
Experimental results of \Hier with several state-of-the-art deep neural nets on \cifar and \imagenet are presented to validate our analysis and show its superiority.
\end{abstract}

\section{Introduction}

Since current deep learning applications such as video action
recognition and speech recognition with huge inputs can take days
even weeks to train on a single GPU, efficient parallelization at
scale is critical to accelerating training of such longtime running machine learning applications. 
Instead of using the classical stochastic gradient descent (SGD)
algorithm originated from the seminal paper by \cite{robbins1951stochastic} as a solver, a number of
parallel and distributed stochastic gradient descent algorithms have
been proposed during the past decade (e.g., see
\cite{zinkevich2010parallelized,recht2011hogwild,dean2012large,dekel2012optimal}). The
first synchronous parallel SGD \cite{zinkevich2010parallelized} is a
naive parallelization of the sequential mini-batch SGD. Global
reductions (averaging) after each local SGD step can incur costly communication
overhead when the number of learners is large. The scaling of
synchronous SGD is fundamentally limited by the batch size. 
Asynchronous SGD (ASGD) algorithms such as
\cite{recht2011hogwild,dean2012large,dekel2012optimal} have
been popular recently for training deep-learning applications.  With ASGD, each learner independently
computes gradients for their data samples, and updates asynchronously
relative to other learners (hence the name ASGD) the parameters
maintained at the parameter server (e.g., see
\cite{dean2012large,li2014scaling}). ASGD algorithms face their own challenges when
the number of learners is large. A single parameter server oftentimes
does not serve the aggregation requests fast enough.
On the other hand, a sharded server though alleviates the aggregation bottleneck but introduces
inconsistencies for parameters distributed on multiple shards. It is
also challenging for ASGD implementations to manage the staleness of gradients which is proprotional to the number
of learners~\cite{li2014scaling}. 

Many recent studies adopt new variants of synchronous parallel SGD
algorithms (see
\cite{hazan2014beyond,johnson2013accelerating,smith2016cocoa,zhang2016parallel,loshchilov2016sgdr,chen2016revisiting,wang2017memory,fan2018kavg}). \cite{fan2018kavg}
analyzed a $K$ step averaging SGD (\KAVG) algorithm, and their
analysis shows that synchrnous parallel SGD with less frequent global averaging 
can sometimes provide faster traning speed and can constantly result in better test accuracies.
Since then a number of variants of \KAVG have been proposed and studied, see \cite{lin2018don,wang2018adaptive} and references therein. 

Although \KAVG demonstrates better scaling behavior than ASGD implementations, the communication cost of global reductions for \KAVG may not be amortized by the local SGD steps when the number of learners is very large. For this reason, we propose a new generic distributed, hierarchical averaging SGD algorithm (\Hier) which can 
reproduce several popular parallel SGD variants by adjusting its parameters.
As \Hier is bulk-synchronous, it allows for infrequent global gradient averaging among learners to effectively minimize
communication overhead just like \KAVG. Instead of using a parameter
server, the learners in \Hier communicate their learned gradients with
their local neighbors at regular intervals for several rounds before global averaging. The
staleness of gradients which can result in divergence of ASGD
methods, can be precisely controlled in \Hier.  Meanwhile, 
it maps well to current and future large
distributed platforms since a single node typically employ multiple GPUs. \Hier intersperse global averaging
with local ones to manage the staleness of
gradients and utilize the natural communtication hierarchy in
the distributed platforms effectively. 

The main contributions of this article are summarized as follows:
1. In section 3.2, we derive several non-asymptotic bounds on the
expected average squared gradient norms for \Hier under different metrics. 
We show that \Hier with infrequent global averaging can still achieve standard convergence rate for non-convex optimization problems. As a byproduct of parallelization, \Hier can deploy larger step size schedule.
2. In section 3.3, we analytically show that \Hier with less frequent global averaging can sometimes have faster training convergence and can constantly have better test accuracy.
3. In section 3.4, by analyzing the bounds we derived, we show that the training speed of \Hier can be improved by deploying more frequent local averaging with more participants.
4. In section 3.5, We compare \Hier with \KAVG and show that local averaging can be used to reduce global averaging
frequency without deterioating traning speed and test accuracy.

The experimental results used to validate our analysis are presented in section 4 on various popular deep neural nets.
To sum up, our analysis and experiments suggest that \Hier with local averaging deployed can use infrequent global reduction, which sheds light on an alternative way to effectively reduce communication overhead without deterioating training speed, and oftentimes provide better test accuracy.

\section{Preliminaries and Notations}
In this section, we introduce some standard assumptions used in the analysis of non-convex optimization algorithms and key notations frequently used throughout this paper. We use $\|\cdot\|_2$ to denote the $\ell_2$ norm of a vector in $\RR^d$; $\langle \cdot \rangle$ to denote the general inner product in $\RR^d$. For the key parameters we use:
\begin{itemize}
	\item $P$ denotes the total number of learners for global averaging.
	\item $S$ denotes the number of learners in a local node for local averaging; we further assume that $S | P$ and $S \geq 1$.
	\item $K_2$ denotes the length of global averaging interval;
	\item $K_1$ denotes the length of local averaging interval and $1 \leq K_1 \leq  K_2$. 
	\item $B_n$ or $B$ denotes the size of mini-batch for the $n$-th global update;
	\item $\gamma_n$ or $\gamma$ denotes the learning rate (step size) for the $n$-th global update;
	\item $\xi^j_{k,s}$ with $j=1,...,P$, $k=1,...,K_2$, and $s=1,...,B$. are i.i.d. realizations of a random variable $\xi$ generated by the algorithm by different learners and in different iterations.
\end{itemize}

We study the following optimization problem:
\begin{equation}
\min\limits_{\bw \in \mathcal{X}} F(\bw)
\end{equation}
where objective function $F:\RR^d \rightarrow \RR$ is continuously differentiable but not necessarily convex over $\mathcal{X}$, and $\mathcal{X}\subset \RR^d$ is a non-empty open subset. Since our analysis is in a very general setting, $F$ can be understood as 
both the expected risk $F(\bw) = \EE f(\bw;\xi) $ or the empirical risk $F(\bw) = n^{-1}\sum_{i=1}^n f_i(\bw)$.  
The following assumptions (see \cite{bottou2018optimization}) are standard to analyze such problems.

\begin{assumption}
	\label{Lipschitz}
	The objective function $F:\RR^d \rightarrow \RR$ is continuously differentiable and the gradient function of $F$ is Lipschitz continuous with Lipschitz constant $L>0$, i.e.
	$$
	\big\| \nabla F(\bw) - \nabla F(\widetilde{\bw})\big\|_2 \leq L\big\| \bw - \widetilde{\bw}\big\|_2
	$$
	for all $\bw$, $\widetilde{\bw}\in \RR^d$. 
\end{assumption}
This assumption is essential to convergence analysis of our algorithm as well as most gradient based ones. Under such an assumption, the gradient of $F$ serves as a good indicator for how far
to move to decrease $F$.

\begin{assumption}
	\label{lowerbound}
	The sequence of iterates $\{\bw_j\}$ is contained in an open set over which $F$ is bounded below by a scalar $F^*$.
\end{assumption}
Assumption \ref{lowerbound} requires that objective function to be bounded from below, which guarantees the problem we study is well defined.
\begin{assumption}
	\label{unbias}
	For any fixed parameter $\bw$, the stochastic gradient $\nabla F(\bw; \xi)$ is an unbiased estimator of the true gradient 
	corresponding to the parameter $\bw$, namely, 
	$$
	\EE_{\xi} \nabla F(\bw; \xi) = \nabla F(\bw).
	$$
\end{assumption}
One should notice that the unbiasedness assumption here can be replaced by a weaker version which is called the First Limit Assumption (see \cite{bottou2018optimization}) that can still be applied to 
our analysis. For simplicity, we just assume that the stochastic gradient is an unbiased estimator of the true one.

\begin{assumption}
	\label{variance}
	There exist scalars $M \geq 0$ such that,
	$$
	\EE_{\xi} \big\| \nabla F(\bw;\xi)\big\|_2^2 - \big\|\EE_{\xi} \nabla F(\bw;\xi) \big\|_2^2 \leq M.
	$$
\end{assumption}
Assumption \ref{variance} characterizes the variance of the stochastic gradients. 

\begin{assumption}
	\label{moment}
	There exist scalars $M_G \geq 0$ such that,
	$$
	\EE_{\xi} \big\| \nabla F(\bw;\xi)\big\|_2^2 \leq M_G.
	$$
\end{assumption}
Assumption \ref{moment} defines a uniform bound on the second order moment of the stochastic gradients.

\section{Main Results}

In this section, firstly we present \Hier as Algorithm 1. \Hier works as follows: each local worker individually runs $K_1$ steps of local SGD; then each group of $S$ workers locally average and synchronize their updated parameter; after each local worker runs a total count of $K_2$ local SGD steps, all $P$ workers globally average and synchronize their parameters and repeat this cycle until convergence. Then we establish the standard convergence results of \Hier and analyze the impact of $K_2$, $S$ and $K_1$ on convergence. Finally, we compare \Hier with \KAVG and show that local averaging can be used to reduce global averaging frequency to achieve communication overhead reduction without deterioating traning speed and test accuracy.

\subsection{\Hier Algorithm}
Assume that $K_2 = K_1 * \beta$ with $\beta \geq 1$. For simplicity of analysis and presentation, we assume that $\beta$ is an integer, which means that 
the length of global averaging interval is multiple of the length of the local one. In practice, it can be implemented at the practitioner's will rather than using $\beta$ as an integer. The performance and results should be consistent with our analysis in this work.
\begin{algorithm}[!htb]
	\caption{\Hier}
	initialize the global parameter $\widetilde{\bw}_1$\;
	\For{$n=1,...,N$ (\textbf{global averaging})}{
		Processor $P_j$, $j=1,\dots,P$ do concurrently:\\
		Synchronize the parameter on each local learner $\bw_{n}^j=\widetilde{\bw}_{n}$ \;
		\For{$b = 0,..., \beta-1$ (\textbf{local averaging})}{
			\For{$k=1,...,K_1$ (\textbf{local SGD})}{
				randomly sample a mini-batch of size $B_n$ and update:
				$$
				\begin{aligned}
				&\bw_{n+b*K_1+k}^j = \bw_{n+b*K_1+k-1}^j \\
				&- \frac{\gamma_n}{B_n} \sum\limits_{s=1}^{B_n} \nabla F(\bw_{n+b*K_1+k-1}^j;\xi_{n+b*K_1+k,s}^j)
				\end{aligned}
				$$
			}	
			Locally average and synchronize the parameters of each worker $P_{j_t}$ within each local cluster:
			$$
			\begin{aligned}
			\bw^{j_t}_{n+(b+1)*K_1} =  \frac{1}{S}\sum\limits_{t=1}^{S} \bw_{n+(b+1)*K_1}^{j_t};
			\end{aligned}
			$$
		}
		Globally average and synchronize 
		$
		\widetilde{\bw}_{n+1} = \frac{1}{P}\sum\limits_{j=1}^P \bw_{n+\beta*K_1}^j
		$\;
	}
	\label{algorithm:1}
\end{algorithm}
One should notice that Algorithm \ref{algorithm:1} is a very general synchronous parallel SGD algorithm. By setting different values of $K_2$, $K_1$ and $S$, it can reproduce various commonly adopted SGD 
variants. For instance, \Hier with $K_2=1$, $K_1=1$ and $S=1$ is equivalent to synchronous parallel SGD \cite{zinkevich2010parallelized}; \Hier with $K_1=1$ and $S=1$ or simply $K_2 = K_1$ is equivalent to \KAVG \cite{fan2018kavg}. 

\subsection{On the Convergence of \Hier}
In this section, we prove the convergence results for Algorithm 1 under two different metrics: one is $T^{-1}\sum_{t=0}^{T-1}\EE \|\nabla F(\bar{\bw}_t) \|_2^2$, where $\bar{\bw}_t = P^{-1}\sum_{j=1}^P w_t^j$ denotes the average weight across all workers at each SGD step. Especially, when $t \equiv 0 \mod K_2$, $\bar{\bw}_t = \widetilde{\bw}_n$ with $n=t/K_2$. Such a metric was used to analyze the convergence behavior 
of \KAVG for strongly convex  \cite{stich2018local} and nonconvex \cite{yu2018parallel} optimization problems. The other metric we use is $N^{-1}\sum_{n=1}^N \EE \big\| \nabla F(\widetilde{\bw}_n)\big\|_2^2$
which only measures the averaged gradient norms at each global update. The former is used to analyze the convergence rate of \Hier and study the impact of global parameter $K_2$ while the latter has a clearer charaterization of the impact of local parameters $K_1$ and $S$ with a more delicate analysis of the local behavior of \Hier.
We derive non-asymptotic upper bounds on the expected average squared gradient norms under constant step size and batch size setting, which serves as a cornerstone of our analysis. Bounds under such a setting are very meaningful and reflects the convergence behavior in real world applications because in practice models are typically trained with only finite many samples, and step size is set as constants during each iteration phase on large distributed plantforms.
\begin{theorem}[fixed step size and fixed batch size]
	\label{theorem: convergence_rate}
	Assume that Assumption 1-5 hold, and Algorithm 1 is run with constant step size $\gamma$ and constant batch size $B$ such that 
	\begin{equation}
	0< L\gamma\leq 1.
	\end{equation}
	Then for all $T\in \NN^*$
	\begin{equation}
	\begin{aligned}
	\frac{1}{T} \sum_{t=0}^{T-1} \EE \big\| \nabla F(\bar{\bw}_t) \big\|_2^2  \leq \frac{2(F(\bar{\bw}_{0}) - F^*) }{\gamma T} + 4L^2  \gamma^2 K_2^2 M_G^2 + \frac{L\gamma M}{PB}.
	\end{aligned}
	\end{equation}
	Especially, by taking 
	\begin{equation}
	\label{stepsize}
	\gamma = \sqrt{PB/T},~{\rm and}~K_2 = T^{1/4}/(PB)^{3/4},
	\end{equation}
	we get
	\begin{equation}
	\label{convergence: rate}
	\frac{1}{T} \sum_{t=0}^{T-1} \EE \big\| \nabla F(\bar{\bw}_t) \big\|_2^2 \leq  \frac{2(F(\bar{\bw}_{0}) - F^*)}{\sqrt{PBT}}  +  \frac{4L^2 M_G^2}{\sqrt{PBT}} + \frac{LM}{\sqrt{PBT}}.
	\end{equation}
\end{theorem}
Proof of Theorem \ref{theorem: convergence_rate} can be found in section \ref{proof: theorem_rate}. When there are $T$ overall updates, globally and locally, a toal number of $PBT$ data samples are processed. 
Theorem \ref{theorem: convergence_rate} indicates that \Hier can achieve an iteration complexity of $O(1/\sqrt{PBT})$ which is the standard convergence rate for nonconvex optimization, see \cite{ghadimi2013stochastic}.
Theorem \ref{theorem: convergence_rate} has two more important implications: 1. to achieve the standard convergence rate, one can use larger step size. From (\ref{stepsize}), we can see that the step size 
$\gamma = O(\sqrt{PB/T})$ is scaled up by the number of workers $P$ as a benefit of parallelization. This result is consistent with the previous analysis of \KAVG in \cite{fan2018kavg}. 2. The length of global averaging 
interval can be large, namely $K_2 = O(T^{1/4}/(PB)^{3/4})$, which indicates that it is unnecessary to use too frequent global averaging. Similar phenomena have been observed for \KAVG and studied by several recent works \cite{fan2018kavg,yu2018parallel,stich2018local}. We will have a more detailed discussion on the behavior of $K_2$ in section \ref{subsection: K_2}.

Although Theorem \ref{theorem: convergence_rate} characterized the convergence rate of \Hier and the impact of global parameter $K_2$, it doesn't capture the behavior of local averaging, namely the impact of local parameters $K_1$ and $S$ on convergence. In the following theorem, we relax Assumption 5 and derive a new upper bound measured by a different metric with a more detailed characterization of $K_1$ and $S$.

\begin{theorem}[fixed step size and fixed batch size]
	\label{theorem: fixed}
	Assume that Assumption 1-4 hold and Algorithm 1 is run with constant step size $\gamma$ and fixed batch size $B$ with the parameters satisfying
	\begin{equation}
	\label{condition: gamma}
	1-  L^2\gamma^2\Big(\frac{K_2(K_2-1)}{2}-1 -\delta_{\nabla F,\bw} \Big)- L\gamma K_2  \geq 0,
	\end{equation}
	Then for all $N \in \NN^*$
	\begin{equation}
	\label{bound:fixed}
	\begin{aligned}
	& \frac{1}{N} \sum\limits_{n=1}^N \EE \big\| \nabla F(\widetilde{\bw}_n)\big\|_2^2  \leq   \frac{2\EE[F(\widetilde{\bw}_{1}) - F^*]}{N(K_2 - \delta)\gamma}   
	+ \frac{L\gamma M K_2^2 }{PB (K_2 - \delta)}  \\ 
	& +  \frac{L^2\gamma^2 M K_2}{12B(K_2 - \delta)} \Big( \frac{(K_2-K_1)(4K_2 + K_1 -3)}{S} + (K_1-1)(3K_2+K_1 -2)\Big)
	\end{aligned}
	\end{equation}
	where $\delta:=  L^2\gamma^2(1 + \delta_{\nabla F,\bw}) \in (0,1)$ and $0< \delta_{\nabla F,\bw} \leq  (K_2-1)K_2/2-1$ is a constant depending on the intermediate gradient norms between each global update.
\end{theorem}

The proof of Theorem \ref{theorem: fixed} can be found in section \ref{proof: theorem1}. Expected (weighted) average squared gradient norms is used as a typical metric to show convergence 
for nonconvex optimization problems, see \cite{ghadimi2013stochastic}. This bound is generic and one can use it to derive classical bounds for different synchronous parallel SGD algorithms by plugging in specific values of $K_2$, $K_1$ and $S$. For example, by plugging in $K_1=1$ and $S=1$ (or simply $K_2=K_1$, in both cases, $K_2$ is $K$ in \KAVG), (\ref{bound:fixed}) reproduce the same bound for \KAVG as in \cite{fan2018kavg}.

As we can see, by only scheduling a constant step size, it converges to some nonzero constant as $N\rightarrow \infty$. 
To make it converge to zero, diminishing step size schedule is needed. Intuitively, more frequent (larger $K_1$) and larger scale (larger $S$) local averaging should lead to faster convergence. 
Bound (\ref{bound:fixed}) justified this intuition for $S$. The impact length of local averaging interval $K_1$ and global averaging interval $K_2$ are more complicated. We will have a more detailed discussion in later sections.

In the following theorem, we show that by scheduling diminishing step size and/or dynamic batch sizes, \Hier converges.
\begin{theorem}[diminishing step size and dynamic batch size]
	\label{theorem: diminishing}
	Assume that Algorithm 1 is run with diminishing step size $\gamma_j$ and growing batch size $B_j$ satisfying 
	\begin{equation}
	1-  L^2\gamma^2_j\Big(\frac{K_2(K_2-1)}{2}-1 -\delta_{\nabla F,\bw} \Big)- L\gamma_j K_2  \geq 0, 
	\end{equation}
	Then for all $N \in \NN^*$
	\begin{equation}
	\label{bound:diminishing}
	\begin{aligned}
	&  \EE \sum\limits_{j=1}^N \frac{\gamma_j}{\sum_{j=1}^N \gamma_j} \big\| \nabla F(\widetilde{\bw}_j)\big\|_2^2  \leq   \frac{2\EE[F(\widetilde{\bw}_{1}) - F^*]}{(K_2-1)\sum_{j=1}^N\gamma_j}   \\
	&+ \sum\limits_{j=1}^N\frac{L M K_2^2 \gamma_j^2}{PB_j (K_2-1)\sum_{j=1}^N \gamma_j }  \\ 
	& + \sum\limits_{j=1}^N \frac{L^2 MK_2 \gamma^3_j}{12B_j(K_2-1)\sum_{j=1}^N\gamma_j}\Big( \frac{(K_2-K_1)(4K_2 + K_1 -3)}{S} \\
	&+ (K_1-1)(3K_2+K_1 -2)\Big) .
	\end{aligned}
	\end{equation}
	Especially, if 
	\begin{equation}
	\label{batchsize}
	\lim\limits_{N\rightarrow \infty} \sum\limits_{j=1}^{N} \gamma_j = \infty,~~ \lim\limits_{N\rightarrow \infty} \sum\limits_{j=1}^N\frac{\gamma_j^2}{PB_j} < \infty,~~\lim\limits_{N\rightarrow \infty} \sum\limits_{j=1}^N \frac{\gamma_j^3}{B_j} < \infty,
	\end{equation}
	Then
	\begin{equation}
	\EE  \sum\limits_{j=1}^N \frac{\gamma_j}{\sum_{j=1}^N \gamma_j}\big\| \nabla F(\widetilde{\bw}_j)\big\|^2_2 \rightarrow 0,~as~N\rightarrow \infty.
	\end{equation}
\end{theorem}
The proof of Theorem \ref{theorem: diminishing} can be found in section \ref{proof: theorem2}. (\ref{batchsize})
also indicates that \Hier can use larger step size schedule than \ASGD which requires $\sum_{j=1}^{\infty} \gamma_j = \infty$, $\sum_{j=1}^{\infty} \gamma_j^2 < \infty$ in general. This result is consistent with our previous analysis in Theorem \ref{theorem: convergence_rate}.

\subsection{Larger Value of $K_2$ Can Sometimes Lead to Faster Training Speed}
\label{subsection: K_2}
As we have shown in Theorem \ref{theorem: convergence_rate}, to achieve the standard convergence rate $O(1/{\sqrt{PBT}})$, the length of global averaging interval $K_2$ can be as large as $T^{1/4}/(PB)^{3/4}$.
In this section, we study the impact of $K_2$ on convergence through the non-asymptotic bound (\ref{bound:fixed}). We show that sometimes larger $K_2$ can even lead to faster training convergence. The result also
implies that adaptive choice of $K_2$ may be better for convergence.

We consider a situation where $T=N*K_2$ is a constant, which means a fixed amount of data is processed or a fixed number of epoches is run.
$K_2$ denotes the length of global averaging interval, or in other words, $K_2$ controls the frequency of global averaging under such setting. Larger $K_2$ means less frequent global averaging thus less frequent  updates on
parameter $\bw$. 

In the following theorem, we analytically show that under certain conditions larger value of $K_2$ can make training process converge faster. This is quite counter intuitive.
Since one might think that smaller $K_2$ (or more frequent global averaging equivalently) should lead to better convergence performance. 
Especially, when $K_2=1$, \Hier is equivalent to sequential SGD with a large mini-batch size. However, it has been shown both analytically and experimentally by several recent works that 
\cite{fan2018kavg,zhang2016parallel,lin2018don,yu2018parallel,wang2018adaptive,stich2018local} that \KAVG with less frequent global averaging sometimes leads to faster convergence and better test accuracy simultaneously.
\begin{theorem}
	\label{theorem: K_2}
	Let $T=N*K_2$ be a constant. Suppose that Algorithm \ref{algorithm:1} is run under the condition of Theorem \ref{theorem: fixed} with fixed $K_1$ and $S$.
	If
	\begin{equation}
	\label{condition:optimal_K}
	\frac{\delta(F(\widetilde{\bw}_1)-F^*)}{T\gamma (1-\delta)}  > \frac{2 L\gamma M}{P B}  + \frac{L^2\gamma^2M}{ BS},
	\end{equation}
	Then \Hier with some $K_2>1$ can have faster training speed than $K_2=1$.
\end{theorem}
The proof of Theorem \ref{theorem: K_2} can found in section \ref{proof: theoremK2}. It essentially says sometimes frequent global averaging is unnecessary for \Hier to gain faster training speed.
This is very meaningful for training large scale machine learning applications. Because global synchronization can cause expensive communication overhead on large platforms. 
As a consequence, the real run time of training can be severely slower when too frequent global reduction is deployed.
Moreover, empirical observations have constantly shown that less frequent global averaging leads to better test accuracy.

To better understand Theorem \ref{theorem: K_2}, condition (\ref{condition:optimal_K}) implies that larger value of $\big(F(\widetilde{\bw}_1)-F^*\big)$ requires some $K_2>1$ thus longer delay to minimize the bound in (\ref{bound:fixed}). The intuition is that if the initial guess is too far away from $F^*$, then less frequent synchronizations can lead to faster convergence for tranining. 
Less frequent averaging implies higher variance of the stochastic gradient in general. It is quite reasonable to think that if it is still far away from the solution, a stochastic gradient with larger variance may be preferred. 
As we mentioned in the proof, the optimal value of $K_2^*$ depends on quantities such as $L$, $M$, and $(F(\widetilde{\bw}_1)-F^*)$ which are unknown to us in practice. Therefore, to obtain a concrete $K_2^*$ in practice is not so realistic.

Corresponding experimental results to validate our analysis are shown in section \ref{section: experiment of K_2}. In that section, we also empirically show that larger $K_2$ can constantly 
provide better test accuracies on various models.

\subsection{Small $K_1$ and Large $S$ can Acceletate Training}
\label{subsection:K1S}
In this section, we study the behavior of two local parameters $K_1$ and $S$, which control the frequency and the scope of local averaging respectively.
Apparently, smaller $K_1$ means more frequent local averaging, and larger $S$ means more number of learners involved in local averaging.
In the following theorem, we show that when $K_2$ is fixed, smaller $K_1$ and larger $S$ can lead to faster training convergence for \Hier.
\begin{theorem}
	\label{theorem: K_1}
	Suppose that Algorithm 1 is run under the same condition as in Theorem \ref{theorem: fixed} or Theorem \ref{theorem: diminishing} with fixed $K_2$. Then both bounds in (\ref{bound:fixed}) and (\ref{bound:diminishing}): 1. are monotone increasing with respect to $K_1$; 2. are monotone decreasing with respect to $S$.
\end{theorem}
The proof of Theorem \ref{theorem: K_1} can be found in section \ref{proof: theoremK1}.
The behaviors of $K_1$ and $S$ are quite expected. It means that more frequent local averaging and/or more participants in local averaging can lead to faster convergence for training. Modern high performance computing (HPC) architechures typically employ multiple GPUs per node and the communication bandwidth within a node is much bigger. Thus the communication cost raised by local averaging can be much less costly than that of global averaging. 

To better understand the impact of local averaging on
convergence, we take a closer look at both bounds (\ref{bound:fixed}) and (\ref{bound:diminishing}). Both $S$ and $K_1$
appear in the third term on the right hand side. When the first part in the third term is dominant, $S$ acts as a scaling factor in $(K_2-K_1)(4K_2 + K_1 -3)/S$, which can be understood as local averaging with more participants amortizes the cost introduced by global averaging represented by $K_2$; when the second term is dominant, one can simply set $K_1=1$ to cancel off this term. These shed light on an alternative way to speed up traning by deploying local averaging. Meanwhile, another lesson we learned here is that one can trade less costly local averaging for global averaging given that less frequent global reduction oftentimes provides better test accuracy and less communication overhead. We will have a more detailed discussion on this in the next section. The experimental results that validate our analysis are presented in section \ref{experiment:K1}. 

\subsection{Using Local Averaging to Reduce Global Averaging Frequency}
\label{section: localtoglobal}
From last section, a meaningful lesson we learned about \Hier is that we can use more local averaging to speed up convergence in the sacrifice of less costly local communications. 
In this section, we compare \Hier with \KAVG, and show that \Hier with less frequent global reduction by deploying local averaging can converge faster than \KAVG while has less communication cost when the 
number of workers $P$ is large.
We consider \Hier and \KAVG in a non-asymptotic scenario where \KAVG is run with $K$ and \Hier with $K_2=(1+a)K$ with some $0<a<1$ and $K_1=1$ and $S=4$. Typically, a single node is equipped with $4$ or more GPUs. Local communication at such a scale is almost negligible. 
Apparently, after processing the same amount of data, \Hier has much less communication cost than \KAVG due to less frequent global averaging involved. 
\begin{theorem}
	\label{theorem: comparison}
	Under the conditions of Theorem \ref{theorem: K_2}, let \Hier be run with $K_2 = (1+a)K$ with $a\in [0,0.6]$, $K_1 =1$ and $S= 4$. Suppose that $L\gamma P\gg 1$.
	Then \Hier can converge faster than \KAVG after processing the same amount of data.
\end{theorem}
The proof of Theorem \ref{theorem: comparison} can be found in section \ref{proof: theorem_comp}.
The result of Theorem \ref{theorem: comparison} has two meaningful consequences: 1. From the point view of parallel computing, even with comparable convergence rate, \Hier with less global averaging whose communication overhead are reduced can have some real run time reduction in the training phase when $P$ is large; 2. As our experimental results show in section \ref{section: increase K2}, less frequent global averaging can often lead to better test accuracy. As a consequence, compared with \KAVG, \Hier can serve as a better alternative algorithm to gain comparable or faster traning speed while achieving better test accuracy. As our experiments show in section \ref{section: increase K2}, we constantly observe that \Hier has better performance than \KAVG even when $a=1$ and $K_1>1$.

\section{Experimental results}
In this section, we present experimental results to validate our analysis of \Hier.
All SGD methods are implemented with Pytorch, and the communication is
implemented using CUDA-aware openMPI 2.0.  All implementations use the
cuDNN library 7.0  for forward and backward propagations. Our experiments are implemented on a
cluster of 32 IBM Minsky nodes interconnected with Infiniband. Each node is an IBM S822LC
system containing 2 Power8 CPUs with 10 cores each, and 4 NVIDIA Tesla P100 GPUs.

We evaluate our algorithm on four state-of-the-art neural network models.
They are \resnet~\cite{he2016deep}, 
\googlenet~\cite{szegedy2015going}, \mobilenet~\cite{howard2017mobilenets}, and \vgg~\cite{simonyan2014very}. They represent some of
the most advanced neural network architectures used
in current large scale machine learning tasks. 
Most of our experiments are done on the dataset \cifar \cite{krizhevsky2009learning}. In addition, we also demonstrate the superior performance of \Hier over \KAVG using the
\imagenet \cite{imagenet_cvpr09} dataset which has a much larger size.
Unless noted, the batchsize we use is $64$, and the total amount of data we train is 200
epochs. The initial learning rate is $0.1$, and decreases to $0.01$ after
$150$ epochs.

\subsection{Impact of $K_2$ on convergence}
\label{section: experiment of K_2}

Theorem \ref{theorem: K_2} shows that
the optimal $K_2$ for convergence is not necessarily $1$, and larger $K_2$
can sometimes lead to faster convergence than a smaller one.  
Fig.~\ref{fig:resnet},~\ref{fig:googlenet},~\ref{fig:mobilenet}
and~\ref{fig:vgg} show the impact of $K_2$ on convergence for
\resnet, \googlenet, \mobilenet, and \vgg respectively.  Within each
figure,  the training accuracies for $K_2=8$, $16$,
and $32$ between epoch 170 to epoch 200 are shown.  We use $P=32$ learners
and set $K_1=4$, $S=4$.

For \resnet and \googlenet, the training accuracies with three
different $K_2$ are similar. In fact, the best training accuracy for
\googlenet is achived with $K_2=32$.  For \mobilenet and \vgg, the
best training accuracies are achieved with $K_2=8$, and the training
accuracy with $K_2=32$ is higher than with $K_2=16$. Above all,
there is no clue that more frequent global averaging (smaller $K_2$) leads to 
faster convergence.

\begin{figure}
	\centering
	\begin{subfigure}[b]{0.23\textwidth}
		\includegraphics[width=0.95\textwidth]{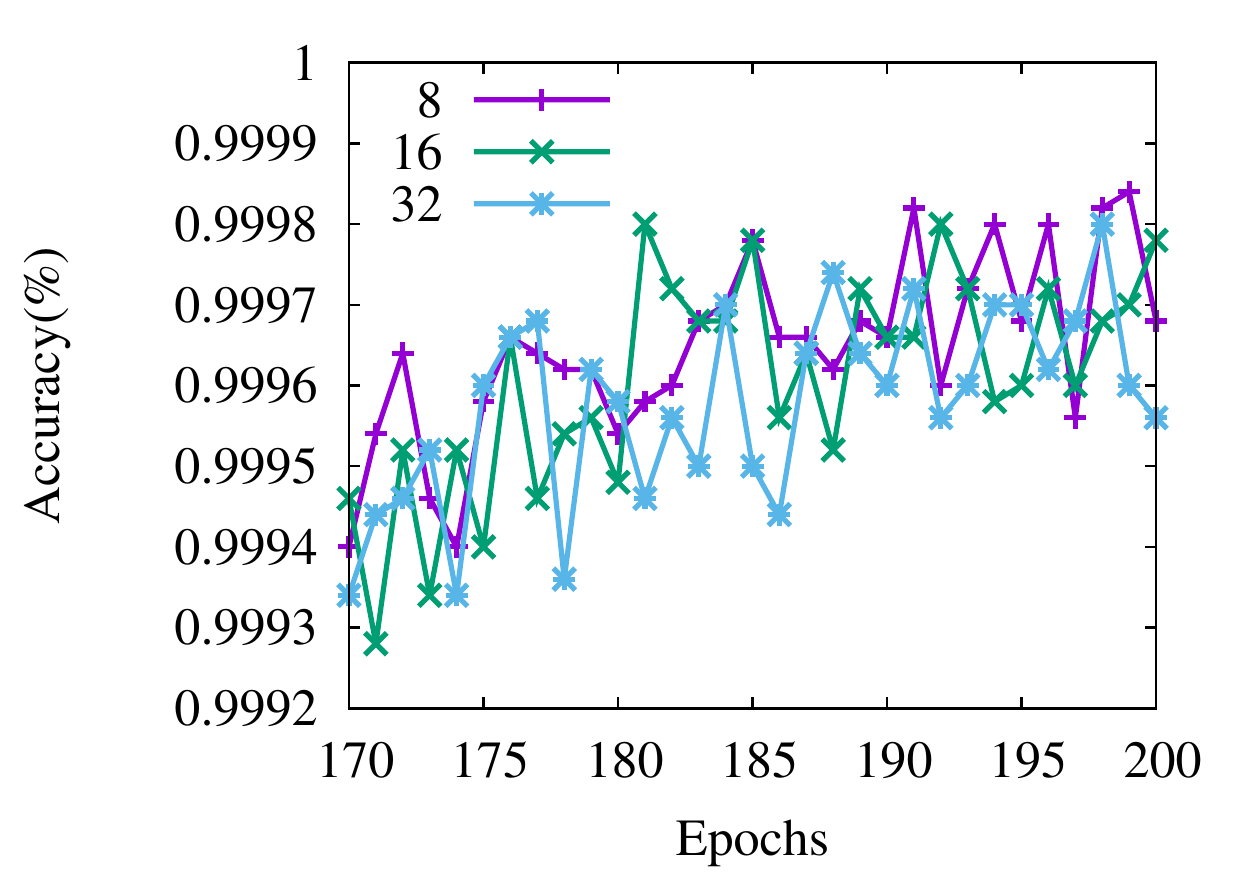}
		\caption{\resnet}
		\label{fig:resnet}
	\end{subfigure}
	~ 
	\begin{subfigure}[b]{0.23\textwidth}
		\includegraphics[width=0.95\textwidth]{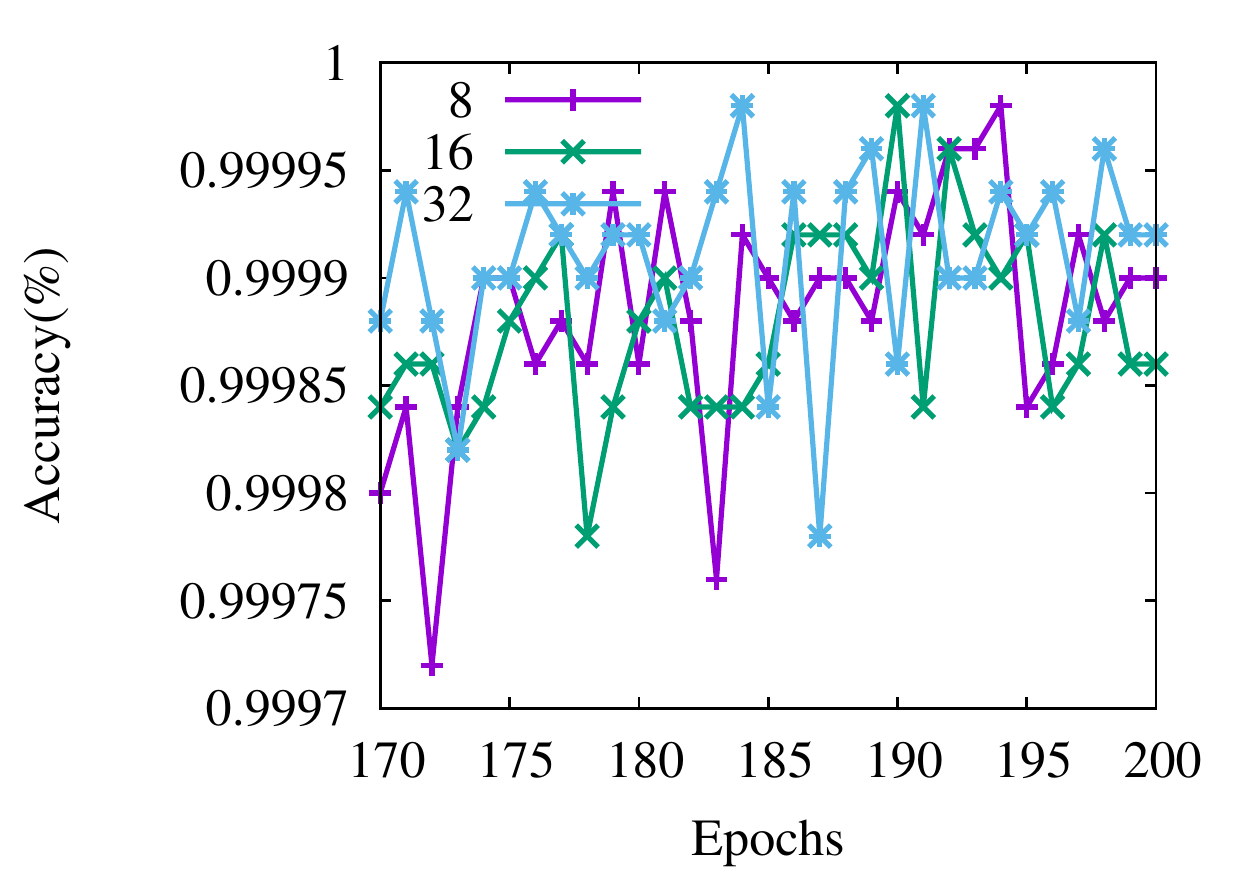}
		\caption{\googlenet}
		\label{fig:googlenet}
	\end{subfigure}
	~ 
	\begin{subfigure}[b]{0.23\textwidth}
		\includegraphics[width=0.95\textwidth]{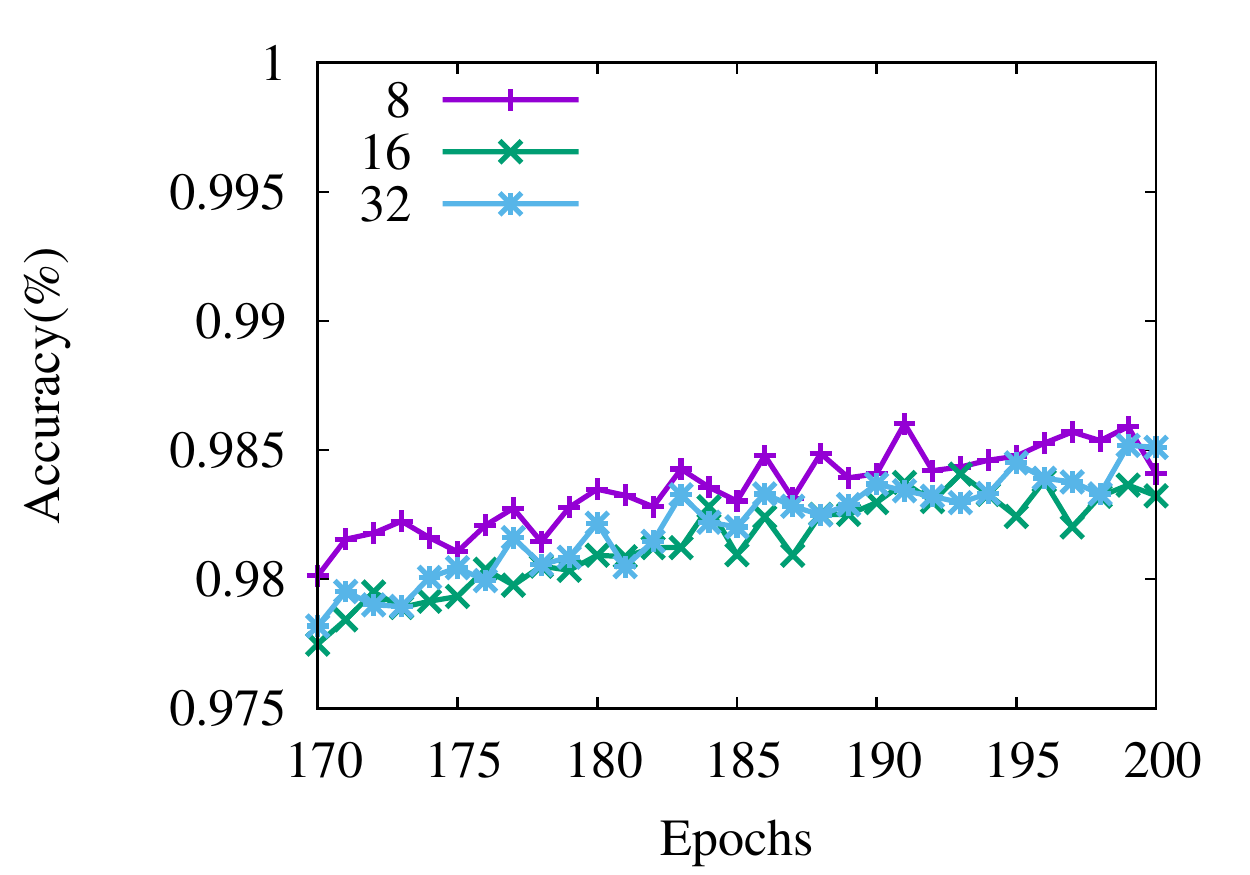}
		\caption{\mobilenet}
		\label{fig:mobilenet}
	\end{subfigure}
		\begin{subfigure}[b]{0.23\textwidth}
			\includegraphics[width=0.95\textwidth]{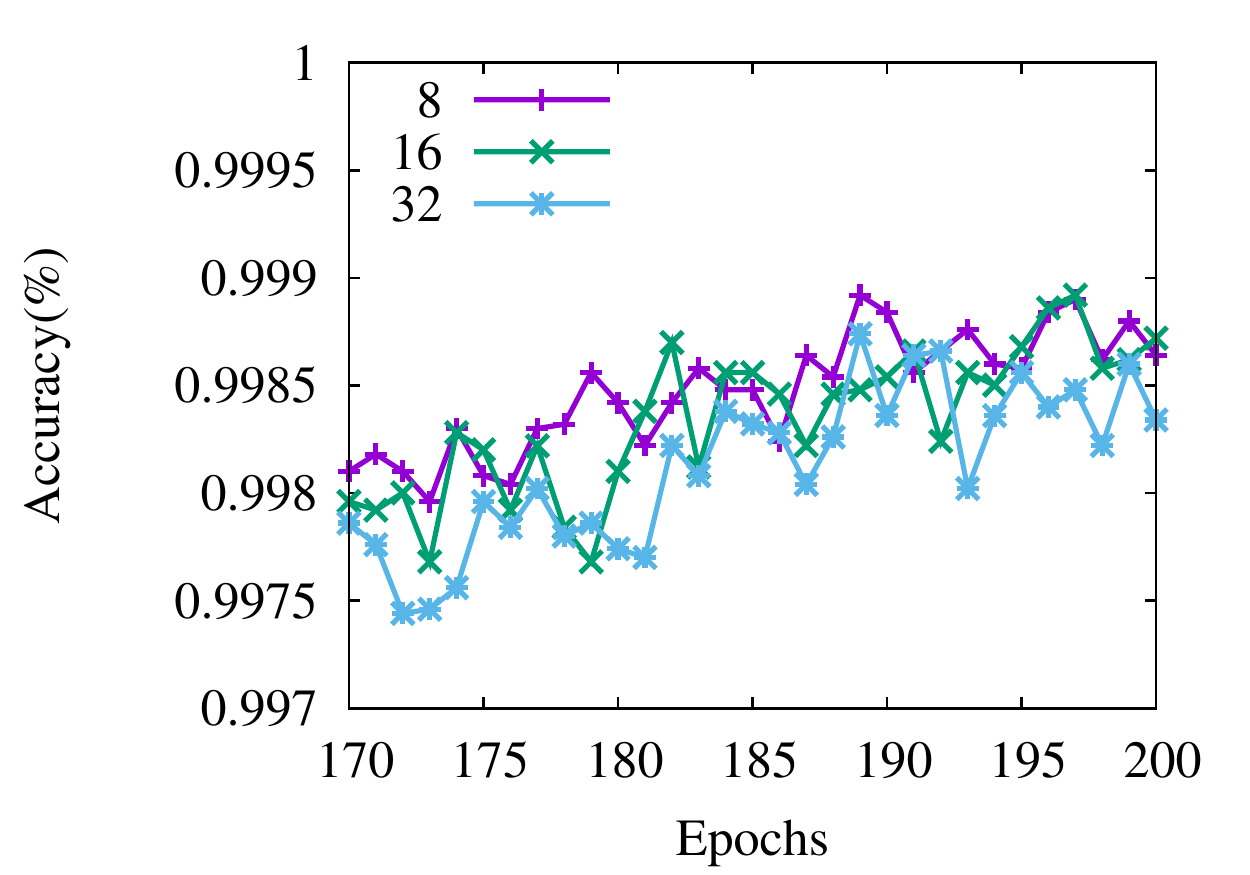}
			\caption{\vgg}
			\label{fig:vgg}
		\end{subfigure}
	\caption{Impact of $K_2$ on convergence: training accuracy}
\end{figure}

Modern neural networks are typically fairly deep and have a large
number of weights. Without mitigation, overfitting can plague
generalization performance. Thus, we also investigate the impact of $K_2$ on test
accuracy.

Fig.~\ref{fig:resnet-val},~\ref{fig:googlenet-val},~\ref{fig:mobilenet-val}
and~\ref{fig:vgg-val} show test accuracies with the same setup for
\resnet, \googlenet, \mobilenet, and \vgg respectively.   For
\resnet, the best test accuracy is achieved with $K_2=16$, about 0.3\%
higher than with $K_2=8$.  For \googlenet, the best test accuracy is
achieved with $K_2=32$, although at epoch 200 all three runs show
similar test accuracy.  For \mobilenet, $K_2=8$, $16$, and $32$ have
similar test performance. For \vgg, $K_2=8$ has the best test accuracy
at epoch 200.

\begin{figure}
	\centering
	\begin{subfigure}[b]{0.23\textwidth}
		\includegraphics[width=\textwidth]{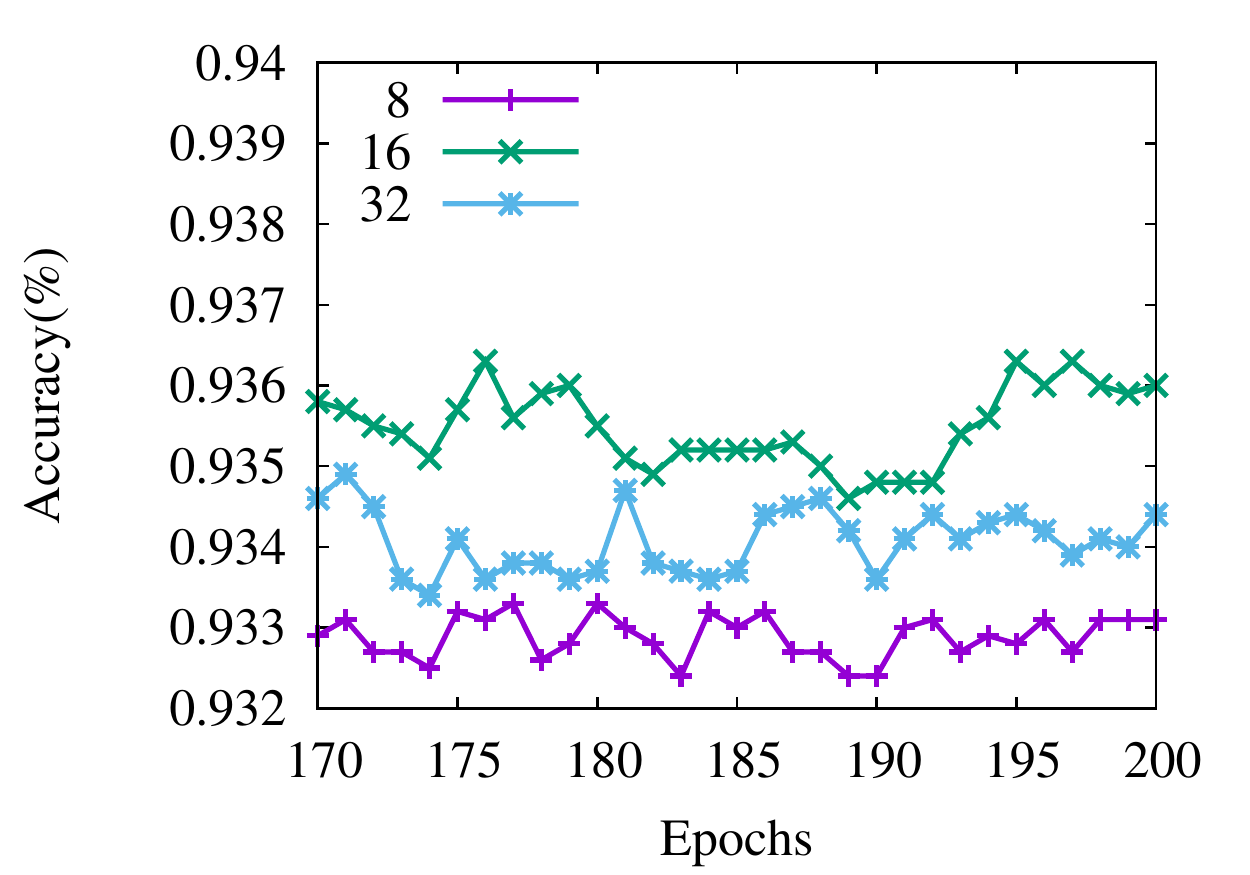}
		\caption{\resnet}
		\label{fig:resnet-val}
	\end{subfigure}
	~ 
	\begin{subfigure}[b]{0.23\textwidth}
		\includegraphics[width=\textwidth]{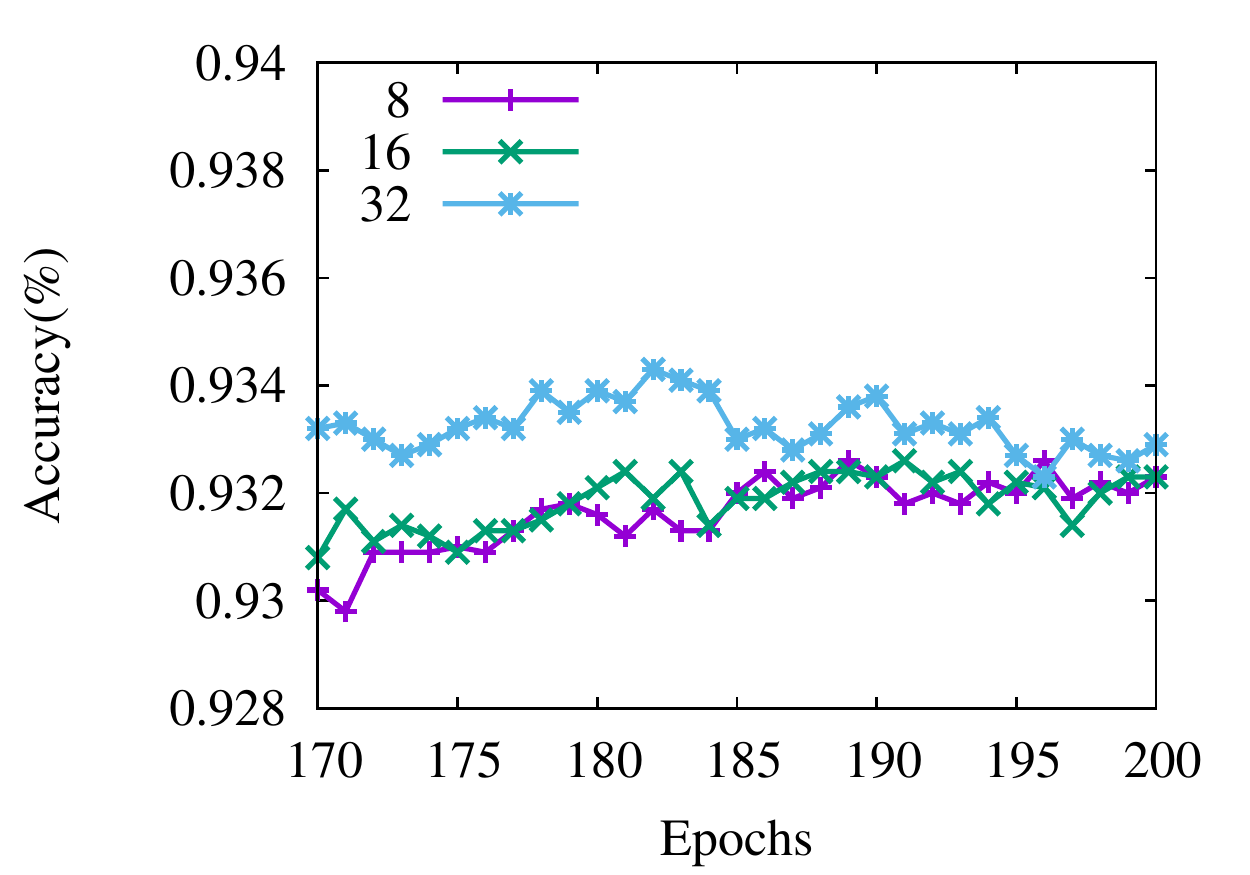}
		\caption{\googlenet}
		\label{fig:googlenet-val}
	\end{subfigure}
	~ 
	\begin{subfigure}[b]{0.23\textwidth}
		\includegraphics[width=\textwidth]{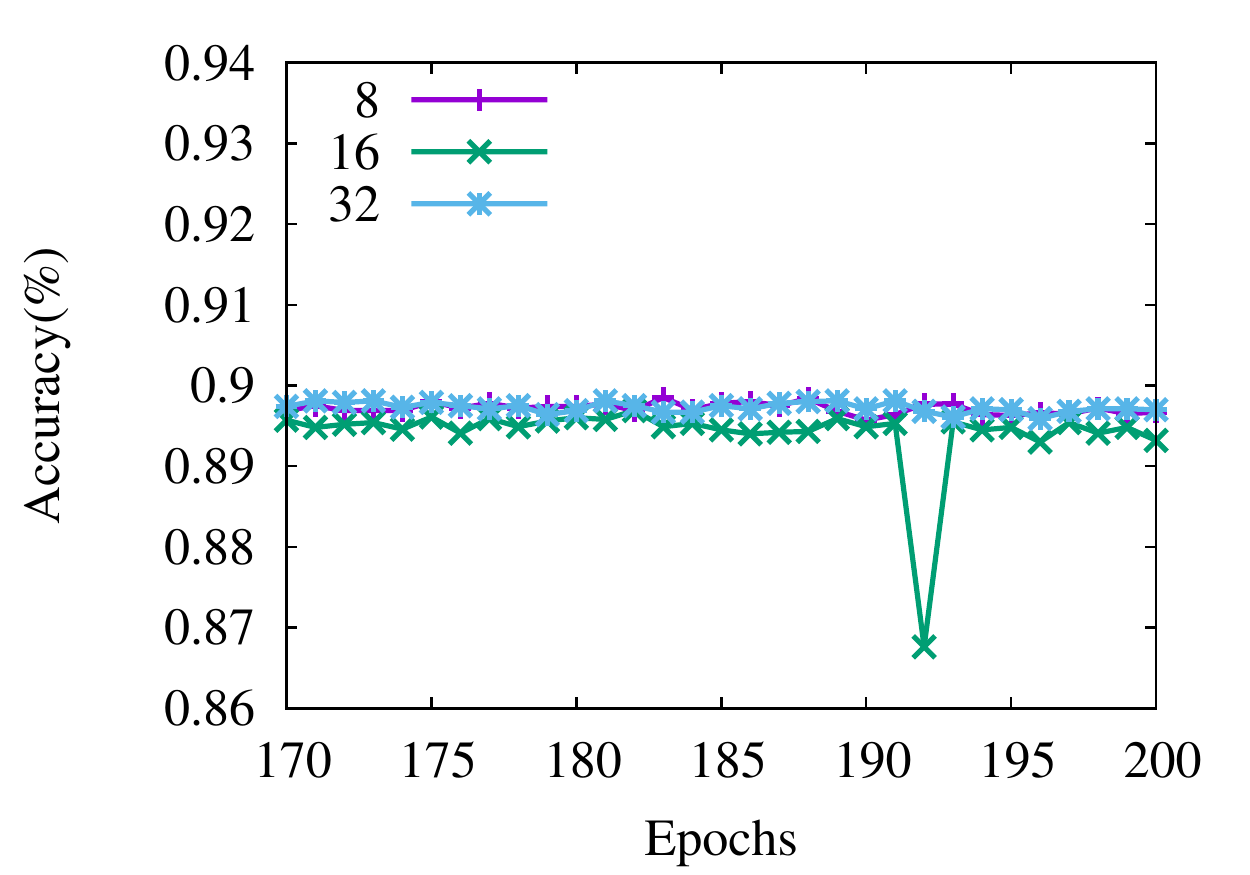}
		\caption{\mobilenet}
		\label{fig:mobilenet-val}
	\end{subfigure}
	\begin{subfigure}[b]{0.23\textwidth}
		\includegraphics[width=\textwidth]{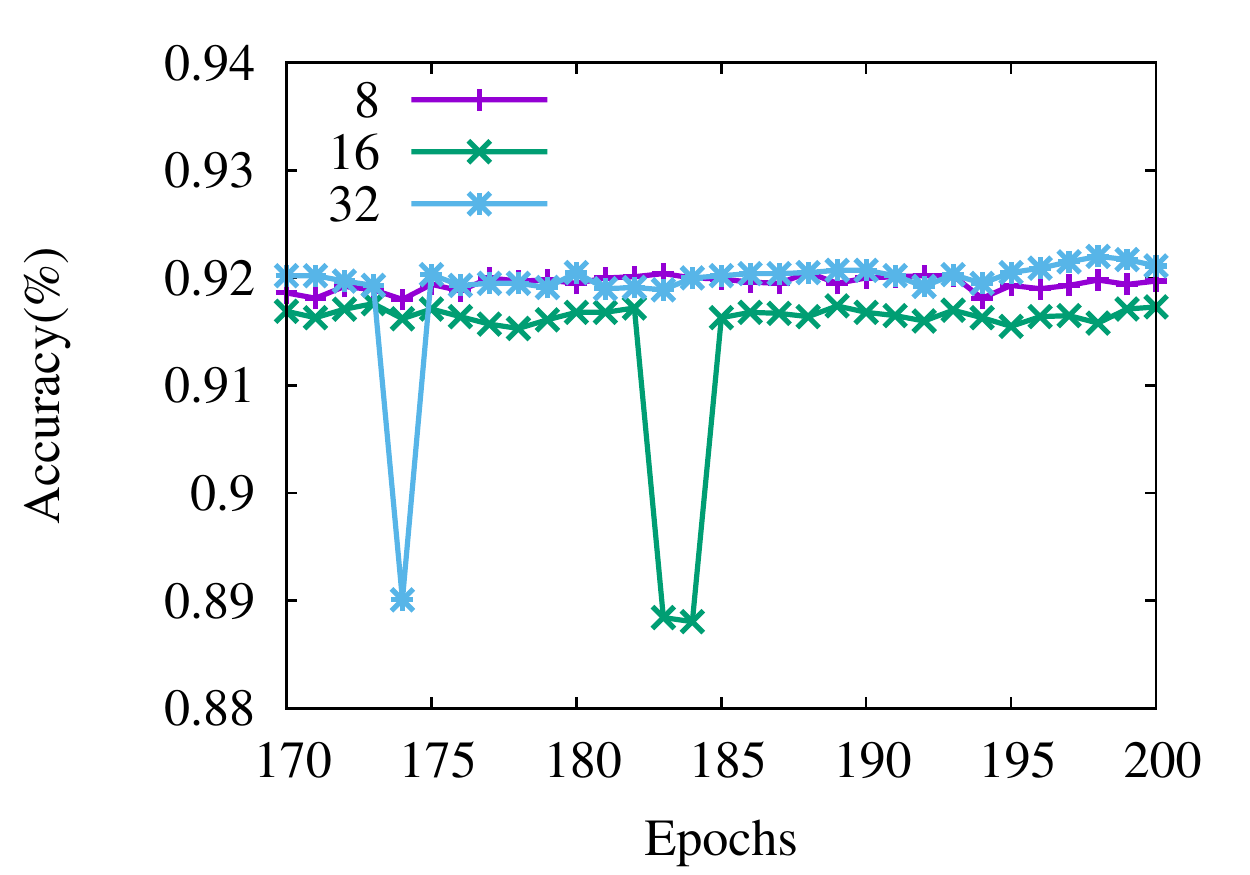}
		\caption{\vgg}
		\label{fig:vgg-val}
	\end{subfigure}
	\caption{Impact of $K_2$ on convergence: test accuracy}
\end{figure}

It is clear that increasing $K_2$ does not necessarily reduce convergence speed for training, 
but obviously it reduces the frequency of costly global reduction when $P$ increases. For example, 
the best test accuracy for \googlenet is achieved with $K_2= 32$. In comparison
with $K_2=8$,  4 times fewer global reductions are used. As a result, the real run time for training can be 
effectively reduced due to much less communication overhead.

\subsection{Impact of  $K_1$ and $S$ on Convergence}
\label{experiment:K1}

In section~\ref{subsection:K1S}, Theorem~\ref{theorem: K_1}  claims that reducing $K_1$ and increasing
$S$ can speed up training convergence.  In practice, with a limited budget in terms of the amount of data
samples processed (e.g., a fixed number of training epochs), we can adjust $K_1$ and $S$ to accelerate
training.  Recall that $K_1$ and $S$ determine local communication
behavior. They provide deterministic means, at least in theory, for practitioners to fine
tune training to achieve the best results within their computational budget and time constraint. 

Fig.~\ref{fig:resnet-k1},~\ref{fig:googlenet-k1},~\ref{fig:mobilenet-k1}
and~\ref{fig:vgg-k1} show the impact of $K_1$ on convergence. As all 
networks achieve high training accuracy, we show the evolution of
trainig loss from epoch 170 to epoch 200. 
In each figure we show the training loss for $K_1=4$ and $8$, and we
set $K_2=32$, $S=4$, and $P=16$. 
As we can see, for all networks it is clear that a lower
training loss is achieved with $K_1=4$ than with $K_1=8$.

\begin{figure}
	\centering
	\begin{subfigure}[b]{0.23\textwidth}
		\includegraphics[width=\textwidth]{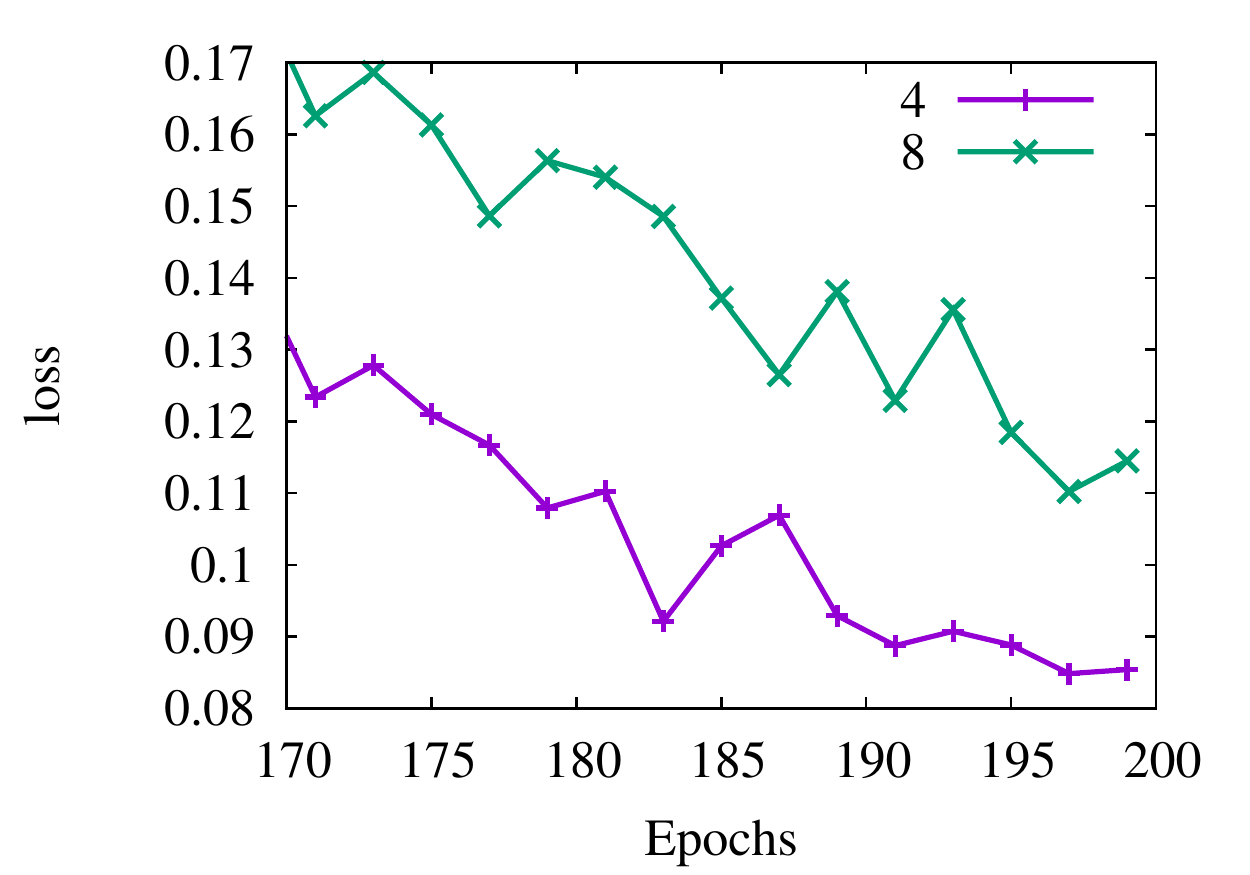}
		\caption{\resnet}
		\label{fig:resnet-k1}
	\end{subfigure}
	~ 
	\begin{subfigure}[b]{0.23\textwidth}
		\includegraphics[width=\textwidth]{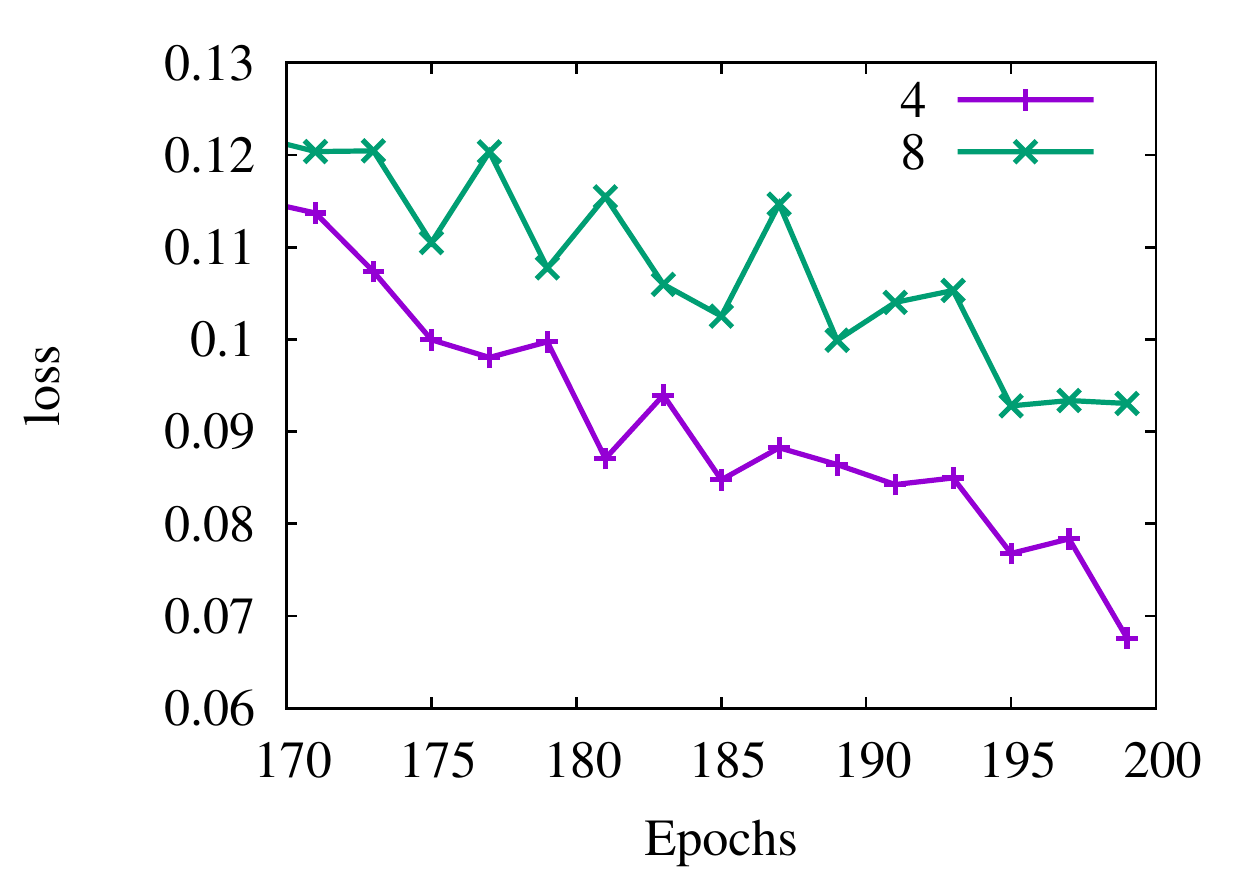}
		\caption{\googlenet}
		\label{fig:googlenet-k1}
	\end{subfigure}
	~ 
	\begin{subfigure}[b]{0.23\textwidth}
		\includegraphics[width=\textwidth]{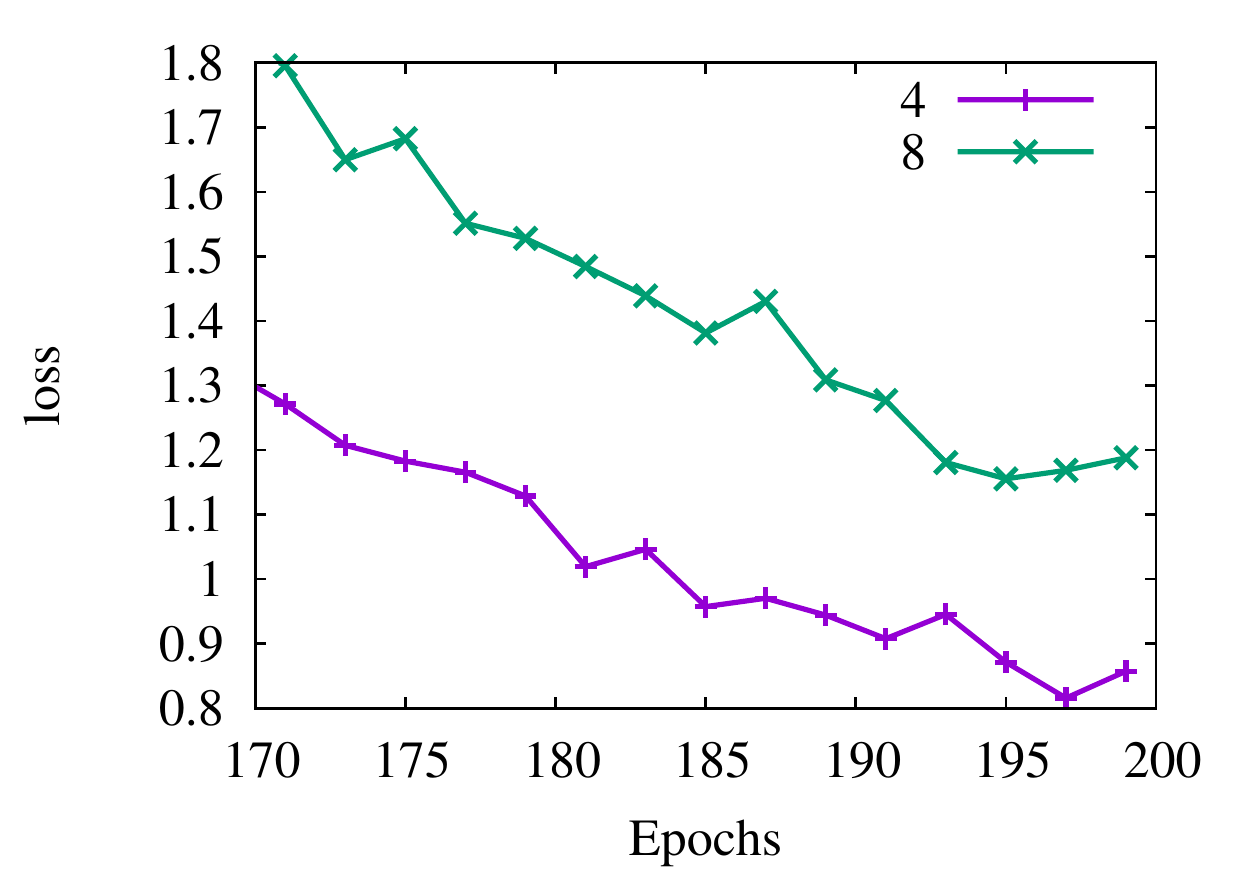}
		\caption{\mobilenet}
		\label{fig:mobilenet-k1}
	\end{subfigure}
	\begin{subfigure}[b]{0.23\textwidth}
		\includegraphics[width=\textwidth]{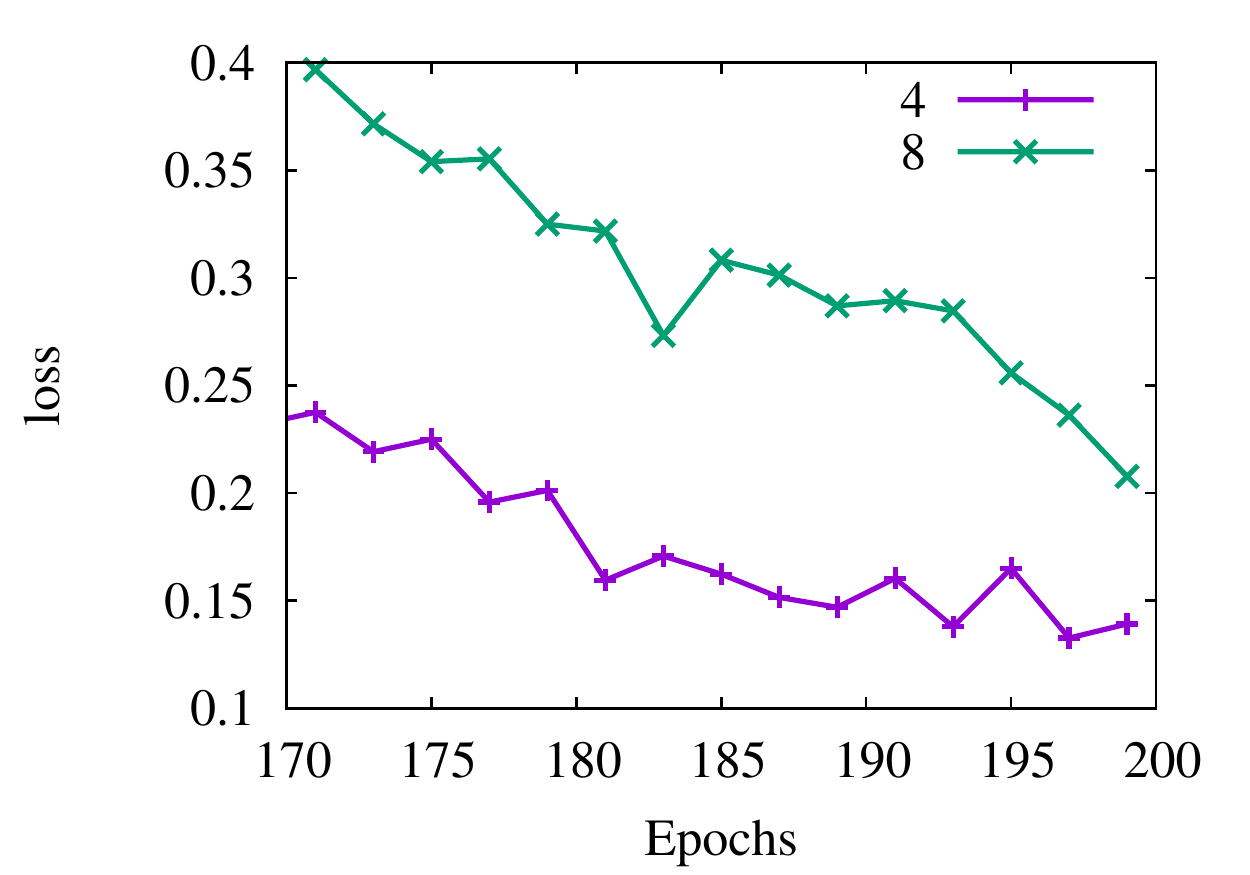}
		\caption{\vgg}
		\label{fig:vgg-k1}
	\end{subfigure}
	\caption{Training loss with $K_1=4$ and $K_1=8$}
\end{figure}

Fig.~\ref{fig:resnet-S},~\ref{fig:googlenet-S},~\ref{fig:mobilenet-S}
and~\ref{fig:vgg-S} show the impact of $S$ on convergence. Again we show the evolution of
trainig loss from epoch 170 to epoch 200. 
In each figure we plot the training loss for $S=2$ and $4$, and we
set $K_2=32$, $K_1=4$, and $P=16$.  In all figures lower training loss
is achieved with $S=4$ than with $S=2$. 

\begin{figure}
	\centering
	\begin{subfigure}[b]{0.23\textwidth}
		\includegraphics[width=\textwidth]{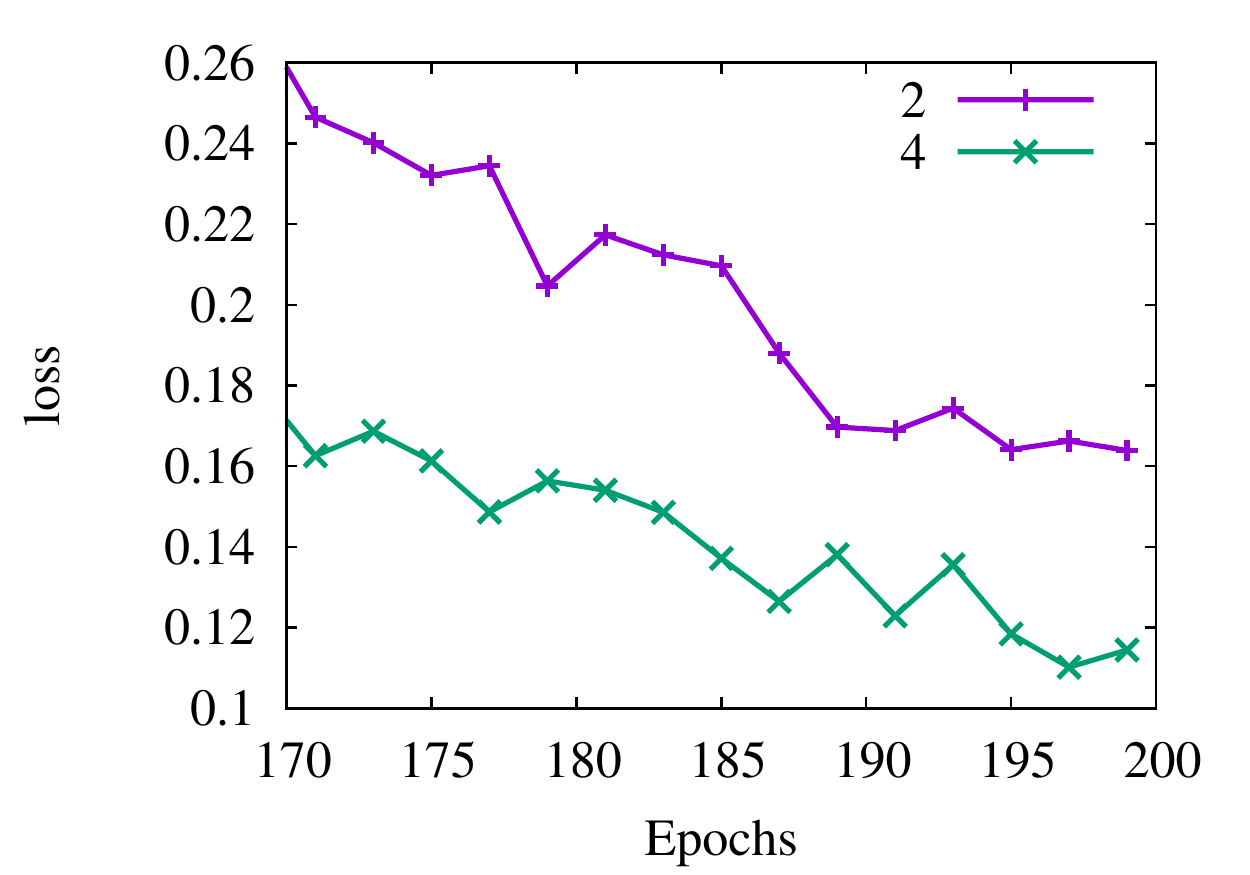}
		\caption{\resnet}
		\label{fig:resnet-S}
	\end{subfigure}
	~ 
	\begin{subfigure}[b]{0.23\textwidth}
		\includegraphics[width=\textwidth]{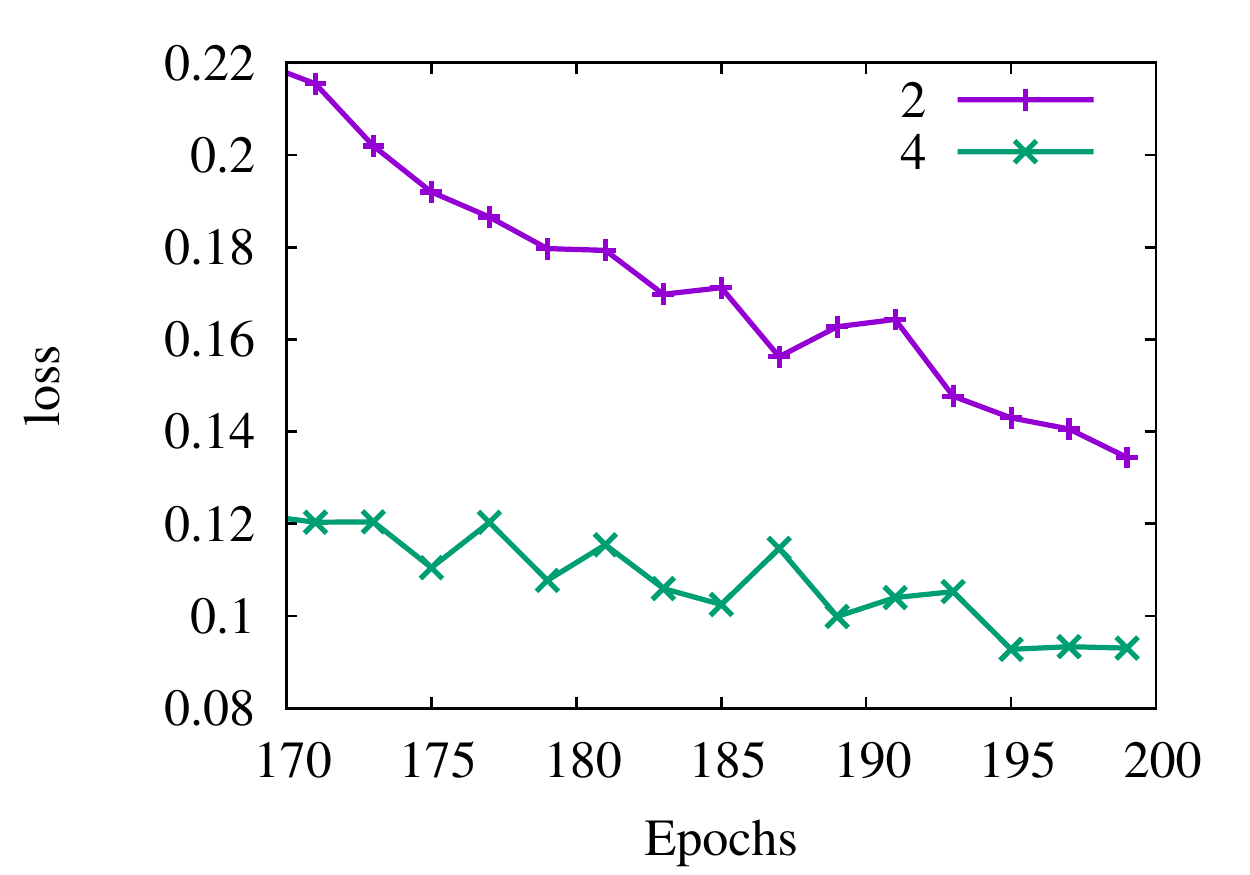}
		\caption{\googlenet}
		\label{fig:googlenet-S}
	\end{subfigure}
	~ 
	\begin{subfigure}[b]{0.23\textwidth}
		\includegraphics[width=\textwidth]{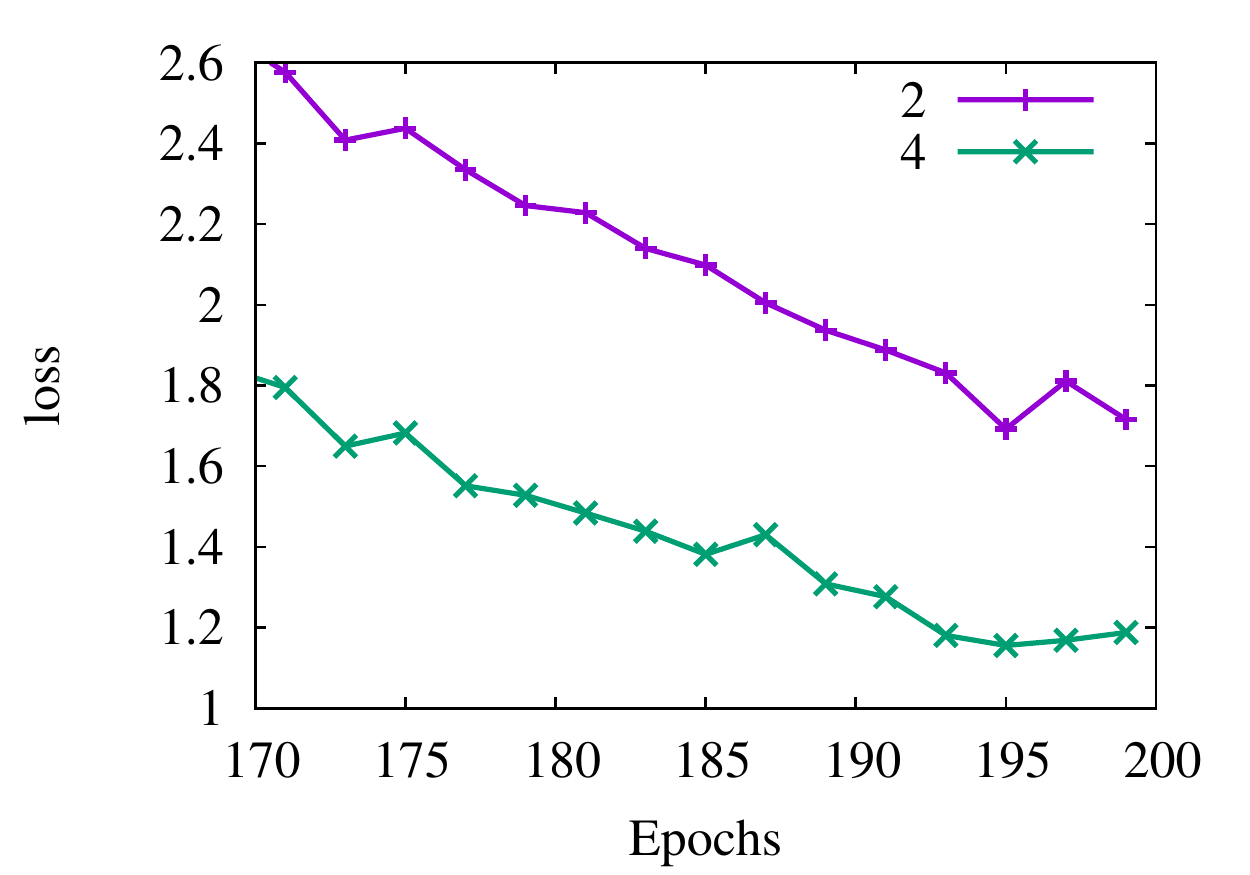}
		\caption{\mobilenet}
		\label{fig:mobilenet-S}
	\end{subfigure}
	\begin{subfigure}[b]{0.23\textwidth}
		\includegraphics[width=\textwidth]{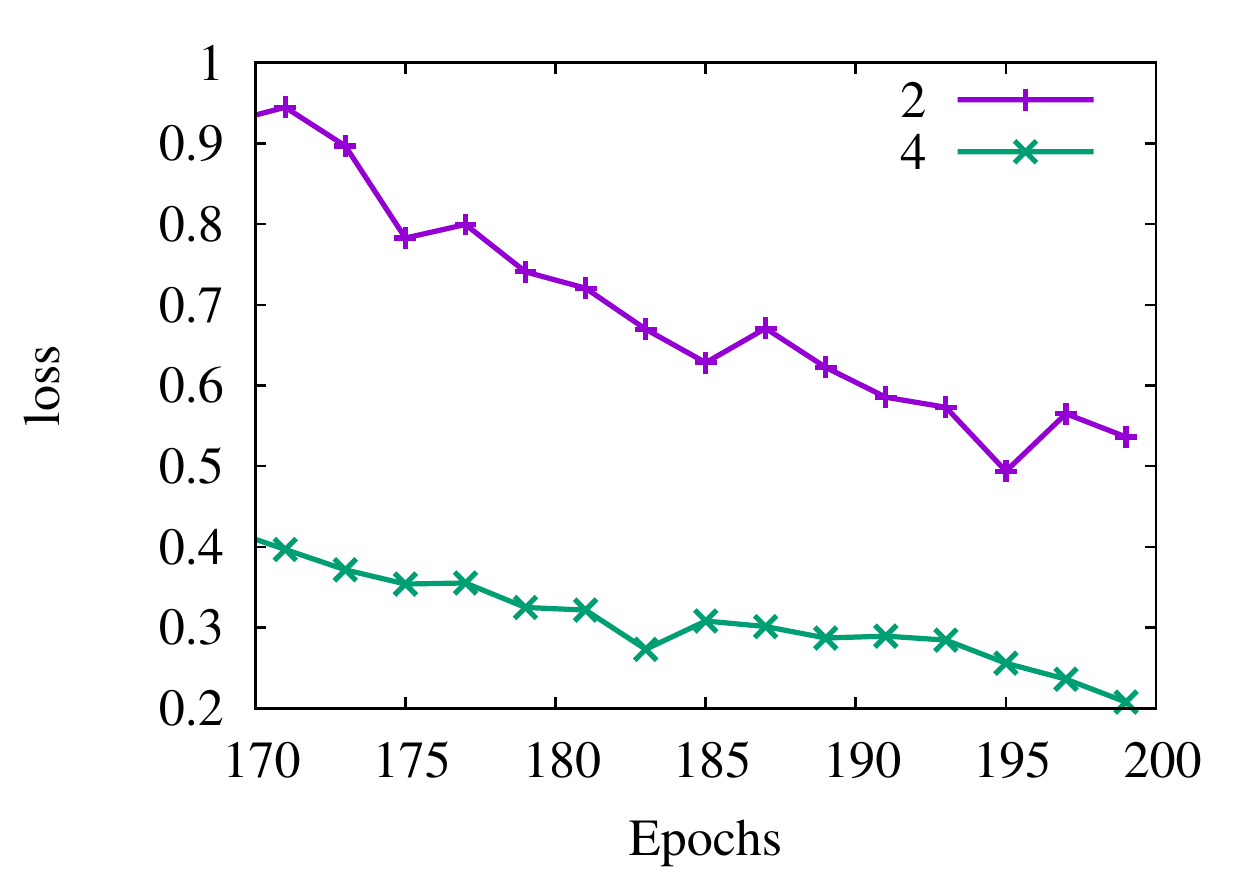}
		\caption{\vgg}
		\label{fig:vgg-S}
	\end{subfigure}
	\caption{Training loss with $S=2$, $4$}
\end{figure}

\subsection{Comparison with \KAVG}
\label{section: increase K2}

As we have mentioned, one of the biggest challenges of distributed training is the
communication overhead. In \KAVG, $K$ determines the frequency of
global reduction. It is shown by \cite{fan2018kavg}, from the perspective of
convergence, large $P$ may require small $K$ for faster
convergence. We explained in section \ref{section: localtoglobal} that \Hier provides the option to
reduce global reduction frequency by using local averaging. Since modern architectures typically employ
multiple GPUs per node, and the intra-node communication bandwidth is
much higher than inter-node bandwith, \Hier is a perfect match for such systems. 

We evaluate the performance of \Hier by setting $K_2=2K_{opt}$ and $S=4$, where $K_{opt}$ is
the fine tuned value of $K$ for \KAVG implementation. The experimental results is summarized in Table \ref{table: comparison}.
We experiment with
$P=16$, $32$, and $64$ learners on \resnet. 
With $16$ learners, $K_{opt}=32$ for \KAVG.  Then we set $K_2=64$ for
\Hier, and experiment with $K_1=2$, $4$, and $16$.  The
corespoinding validation accuracies are $94.01\%$, $94.11\%$, and $94.08\%$
respectively. They are all higher than the best accuracy achieved by
\KAVG at $94.0\%$.
With 32 and 64 learners, $K_{opt}=4$ for \KAVG.  We set $K_2=8$ for
\Hier, the accuracies achieved are $93.90\%$ and $93.17\%$ at $K_1=4$, $S=8$
and $K_1=1$ $S=4$, respectively. The best accuracies achieved by \KAVG with
32 and 64 learners are $93.7\%$ and $92.5\%$ respectively. 

In our experiments, while reducing the gobal reduction frequency by
half, \Hier still achieves validation accuracy comparable to
\KAVG.  Note that we do not show the actual wallclock time per epoch because
Pytorch implementations do not support GPU-direct communication
yet on our target architecture. For all reductions, the data is copied from GPU to CPU first. It is clear though once
GPU-direct communication is implemented,  \Hier can effectively reduce communication time.

\begin{table}
	\begin{center}
		\begin{tabular}{|c|c|c|c|c|c|c|}
			\hline \hline
			Alg.     & $K_{opt}$ & $K_2$ & $K_1$ & $S$ & $P$ & Test accuracy \\ \hline \hline
			\KAVG & 32   & -         & -         & -     &  16   &  $94.00\%$       \\ 
			\Hier    & -      &  \textbf{64}     & 2         & 4     &  16   &  $\textbf{94.01\%}$      \\
			\Hier    & -      &  \textbf{64}     & 4         & 4     &  16   &  $\textbf{94.11\%}$       \\ 
			\Hier    & -      &  \textbf{64}     & 16         & 4     &  16   &  $\textbf{94.08\%}$       \\  \hline
			\KAVG & 4     & -         & -         & -     &  32   &  $93.70\%$       \\ 
            \Hier    & -     &  \textbf{8}       & 4         & 8     & 32   &  $\textbf{93.90\%}$       \\ \hline 
            \KAVG & 4     & -         & -         & -     &  64   &  $92.50\%$       \\ 
            \Hier    & -     &  \textbf{8}        & 1         & 4     & 64   &  $\textbf{93.17\%}$       \\ \hline \hline
		\end{tabular}
	\caption{Comparison of \Hier and \KAVG}
		\label{table: comparison}
	\end{center}
\end{table}

\subsection{Performance of \Hier on ImageNet}
In this section, we further investigate the performance of \Hier with the ImageNet-1K dataset which is much larger than
\cifar and it contains of 1.28 million
training images split across 1000 classes, and 50,000 validation images. 

During training, a crop of random size (of $0.08$
to 1.5) of the original size and a random aspect ratio (of
3/4 to 4/3) of the original aspect ratio is made. This crop is then
resized to $224\times224$. Random color jittering with a ratio of 0.4 to the brightness, contrast
and saturation of an image is then applied.  Next a random
 horizontal flip is applied to the input, and the input is then
 normalized with mean (0.485, 0.456, 0.406) and standard deviation
 (0.229, 0.224, 0.225) for the (R, G, B) channels respectively. 
 For \KAVG we set $K=43$, and for \Hier we set $K_2=43$,  $K_1=20$,
 and $S=4$. 
 
 Fig.~\ref{fig:imagenet-train} shows the training accuracies comparison between \KAVG and \Hier with 16 learners. 
 Clearly, \Hier achieves higher training
 accuracy than \KAVG since the first epoch. After the first 5 epochs, \Hier achieved
 $6\%$ higher training accuracy than \KAVG, and at the 46-th epoch, \Hier
 achieved $17.33\%$ higher training accuracy than \KAVG. At the 90-th
 epoch, the training accuracy of \Hier is $1.15\%$ higher than \KAVG.

Fig.~\ref{fig:imagenet-val} shows the test accuracies comparison between \KAVG and \Hier with $16$ learners. 
As we can see, \Hier also achieves higher
validation accuracy than \KAVG since the first epoch. At epoch 5, \Hier achieved
$12 \%$ higher accuracy than \KAVG, and at the 90-th epoch, \Hier
achieved $0.51\%$ higher accuracy than \KAVG.

\begin{figure}
	\centering
	\begin{subfigure}[b]{0.4\textwidth}
		\includegraphics[width=\textwidth]{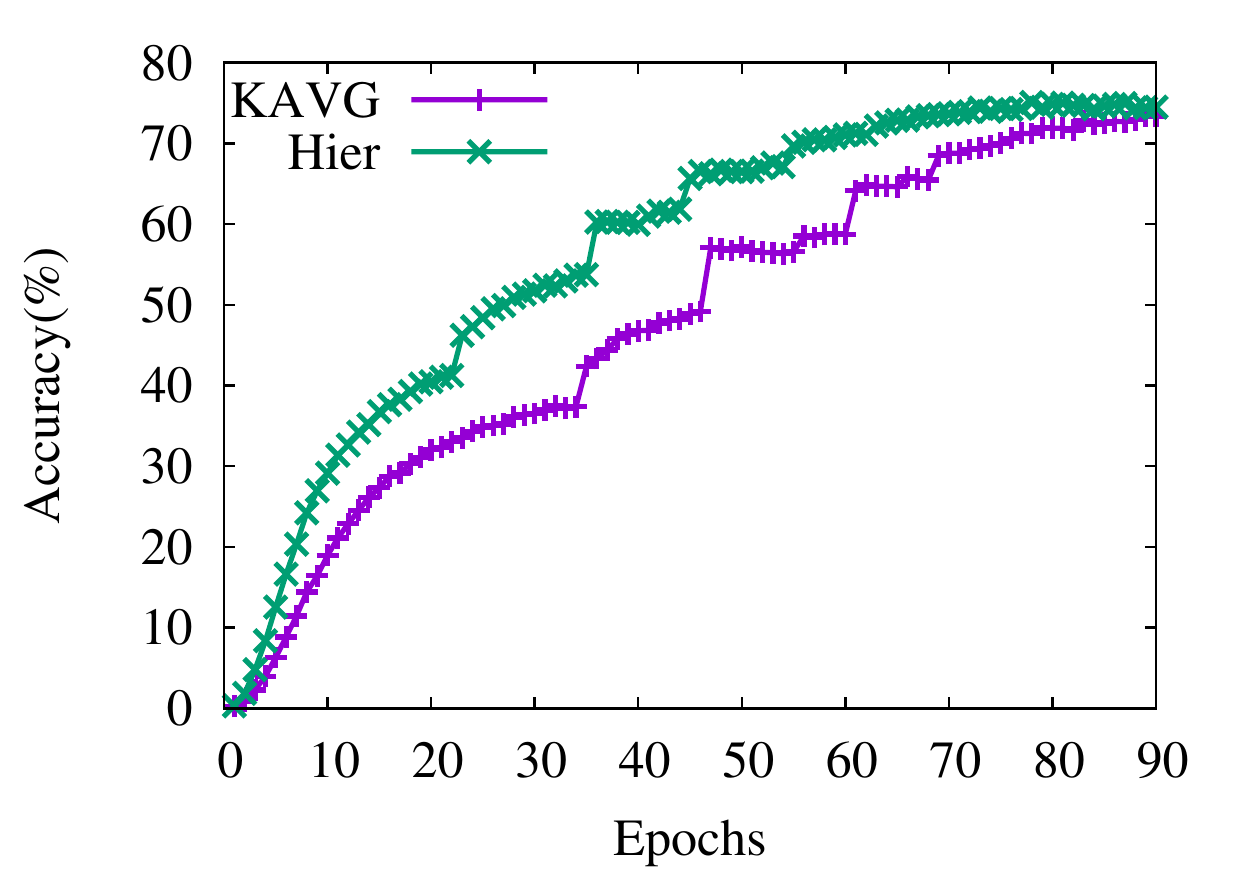}
		\caption{Training}
		\label{fig:imagenet-train}
	\end{subfigure}
	~ 
	\begin{subfigure}[b]{0.4\textwidth}
		\includegraphics[width=\textwidth]{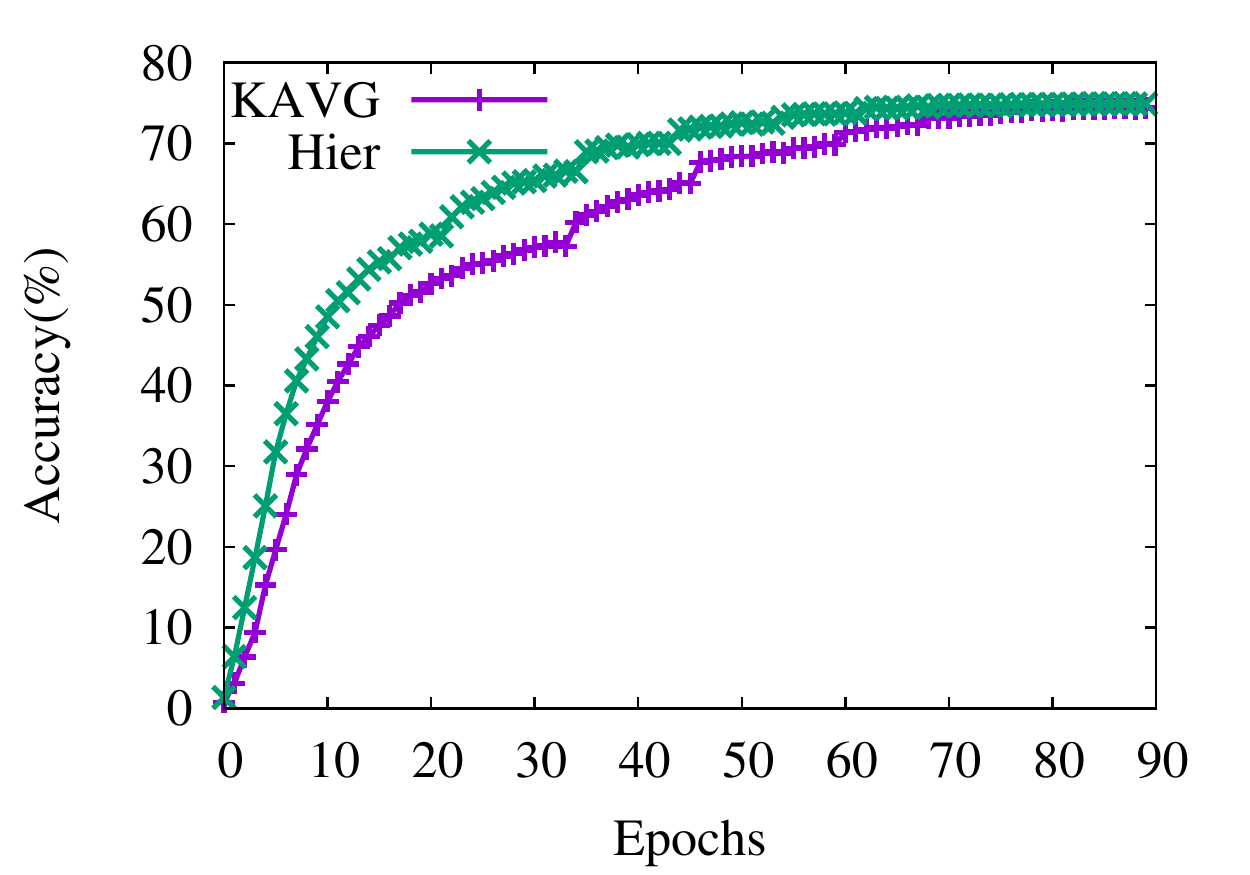}
		\caption{Test}
		\label{fig:imagenet-val}
	\end{subfigure}
	\caption{Performance of \KAVG and \Hier with ImageNet-1K}
\end{figure}

\section{Conclusion}

We proposed a two stage hierarchical averaging SGD algorithm to effectively reduce communication overhead while not deterioate
training and test performance for distributed machine learning. We established the convergence results for \Hier for non-convex optimization problems and show that \Hier with infrequent global reduction can still achieve the expected convergence rate, while oftentimes provides faster training speed and better test accuracy. By introducing local averaging, we show that it can be used to accelerate training. 
Moreover, we show analytically and experimentally that local averaging can serve as an alternative remedy to reduce global reduction frequency without doing harm to the convergence rate for training and generalization performance for testing. As a result, \Hier provides an alternative method for practitioners to train large scale machine learning applications on distributed platforms.

\section{Proofs}

\subsection{Proof of Theorem \ref{theorem: convergence_rate}}
\label{proof: theorem_rate}
\begin{proof}
	We denote by  
	\begin{equation}
	\bar{\bw}_t = \frac{1}{P} \sum_{j=1}^P  \bw^j_t
	\end{equation}
	as the global average of local iterates over all $P$ workers. A quick observation is that when $t  \equiv 0 \mod K_2$, then $\bar{\bw}_t = \widetilde{\bw}_{i}$ with $i=t/K_2$. In the following theorem,
	we derive an upper bound on the convergence measured by $T^{-1}\sum^T_{t=1} \EE\| \nabla F(\bar{\bw}_{t-1})\|^2_2$.
	It is easy to see that 
	\begin{equation}
	\bar{\bw}_{t+1} - \bar{\bw}_t =- \frac{\gamma}{PB} \sum\limits_{j=1}^P \sum\limits_{s=1}^B \nabla F(\bw_{t}^j; \xi^j_{t+1,s}).
	\end{equation}
	Consider 
	\begin{align}
	\EE \Big[F(\bar{\bw}_{t+1}) - F(\bar{\bw}_{t}) \Big] & \leq \EE \Big\langle \nabla F(\bar{\bw}_t), \bar{\bw}_{t+1} - \bar{\bw}_t \Big\rangle + \frac{L}{2} \EE \big\| \bar{\bw}_{t+1} - \bar{\bw}_t \big\|_2^2 \\
	&\leq -\gamma \EE \Big[\Big\langle \nabla F(\bar{\bw}_t), \frac{1}{P} \sum_{j=1}^P \nabla F(\bw^j_{t} )\Big\rangle \Big]   \label{term: cross_1} \\
	&+ \frac{L\gamma^2}{2P^2B^2}  \EE \big\|  \sum\limits_{j=1}^P  \sum\limits_{s=1}^B \nabla F (\bw^j_{t};\xi^j_{t+1,s})\big\|_2^2  \label{term: cross_2}.
	\end{align}
	where (\ref{term: cross_1}) is due to the fact that random variables $\nabla F(\bw_{t}^j; \xi^j_{t+1,s})$ are i.i.d. over $s$ for fixed $t$ and $j$ conditioning on previous steps.
	In the following, we will bound (\ref{term: cross_1}) and (\ref{term: cross_2}) respectively.
	
	For (\ref{term: cross_1}), some simple algebra implies that 
	\begin{equation}
	\begin{aligned}
	&-\gamma \EE \Big[\Big\langle \nabla F(\bar{\bw}_t), \frac{1}{P} \sum_{j=1}^P \nabla F(\bw^j_{t} )\Big\rangle \Big]  \\
	&  =  -\frac{\gamma}{2} \EE \Big[ \| \nabla F(\bar{\bw}_t)\|_2^2  + \big\|\frac{1}{P} \sum_{j=1}^P \nabla F(\bw_t^j)  \big\|_2^2 - \big\|\nabla F(\bar{\bw}_t) -\frac{1}{P} \sum_{j=1}^P \nabla F(\bw_t^j)  \big\|_2^2    \Big]
	\end{aligned}
	\end{equation}
	For term $\big\|\nabla F(\bar{\bw}_t) -\frac{1}{P} \sum_{j=1}^P \nabla F(\bw_t^j)  \big\|_2^2$, we have
	\begin{equation}
	\label{term: 1}
	\begin{aligned}
	\EE\big\|\nabla F(\bar{\bw}_t) -\frac{1}{P} \sum_{j=1}^P \nabla F(\bw_t^j)  \big\|_2^2 &= \frac{1}{P^2} \big\| \sum_{j=1}^P \Big( \nabla F(\bar{\bw}_t) - \nabla F(\bw_t^j)\Big)  \big\|_2^2 \\
	&  \leq \frac{1}{P}   \sum_{j=1}^P  \EE \Big\|  \nabla F(\bar{\bw}_t) - \nabla F(\bw_t^j) \Big\|_2^2 \\
	& \leq \frac{L^2 }{P} \sum_{j=1}^P \EE \| \bar{\bw}_t - \bw_t^j \|_2^2.
	\end{aligned}
	\end{equation}
	To bound $ \EE \| \bar{\bw}_t - \bw_t^j \|_2^2$, we first set $t_0$ to be the largest integer such that $t_0 \equiv 0 \mod K_2 $ and $t_0\leq t$. In other words, $t_0$ is the latest iteration number that is less than $t$ when global averaging happens. Then we can write
	\begin{equation}
	\label{iterate: 1}
	\bw^j_t = \bar{\bw}_{t_0} - \frac{\gamma}{B} \sum_{\tau = t_0+1}^t \sum_{s=1}^B \nabla F(\bw_{\tau-1}^j, \xi^j_{\tau,s}) 
	\end{equation}
	and 
	\begin{equation}
	\label{iterate: 2}
	\bar{\bw}_t = \bar{\bw}_{t_0} - \frac{\gamma}{PB} \sum_{j=1}^P\sum_{\tau = t_0+1}^t \sum_{s=1}^B \nabla F(\bw_{\tau-1}^j, \xi^j_{\tau,s}) .
	\end{equation}
	Plug (\ref{iterate: 1}) and (\ref{iterate: 2}) in $ \EE \| \bar{\bw}_t - \bw_t^j \|_2^2$, we get
	\begin{equation}
	\label{term: diff_1}
	\begin{aligned}
	\EE \big\| \bar{\bw}_t - \bw_t^j  \big\|_2^2  &= \gamma^2 \EE \big\| \sum_{\tau=t_0+1 }^t \Big( \frac{1}{PB} \sum_{j=1}^P \sum_{s=1}^B \nabla F(\bw_{\tau-1}^j, \xi^j_{\tau,s}) - \frac{1}{B}  \sum_{s=1}^B \nabla F(\bw^j_{\tau-1},\xi^j_{\tau,s}) \Big) \big\|_2^2 \\
	& \leq 2\gamma^2  \EE \Big( \big\| \sum_{\tau=t_0+1 }^t  \frac{1}{PB} \sum_{j=1}^P \sum_{s=1}^B \nabla F(\bw_{\tau-1}^j, \xi^j_{\tau,s}) \big\|_2^2  +  \big\| \sum_{\tau=t_0+1 }^t  \frac{1}{B}  \sum_{s=1}^B \nabla F(\bw_{\tau-1}^j, \xi^j_{\tau,s}) \big\|_2^2      \Big) \\
	& \leq 2\gamma^2(t-t_0)   \sum_{\tau=t_0+1 }^t \sum_{s=1}^B \Big( \frac{1}{PB} \sum_{j=1}^P  \EE \big\|  \nabla F(\bw_{\tau-1}^j, \xi^j_{\tau,s}) \big\|_2^2  +  \frac{1}{B}  \EE \big\|  \nabla F(\bw_{\tau-1}^j, \xi^j_{\tau,s}) \big\|_2^2\Big) \\
	& \leq 4\gamma^2 (t-t_0)^2 M_G^2 \\
	& \leq 4 \gamma^2 K_2^2 M_G^2.
	\end{aligned}
	\end{equation}
	Plug (\ref{term: diff_1}) back into (\ref{term: 1}), we get 
	\begin{equation}
	\EE\big\|\nabla F(\bar{\bw}_t) -\frac{1}{P} \sum_{j=1}^P \nabla F(\bw_t^j)  \big\|_2^2 \leq 4 L^2\gamma^2  K_2^2 M_G^2.
	\end{equation}
	Then we get the bound on (\ref{term: cross_1}) as 
	\begin{equation}
	\label{bound: cross_1}
	\begin{aligned}
	-\gamma \EE \Big[\Big\langle \nabla F(\bar{\bw}_t), \frac{1}{P} \sum_{j=1}^P \nabla F(\bw^j_{t} )\Big\rangle \Big] \leq -\frac{\gamma}{2} \EE \Big[ \| \nabla F(\bar{\bw}_t)\|_2^2 
	+ \big\| \frac{1}{P} \sum_{j=1}^P \nabla F(\bw_t^j)  \big\|_2^2 \Big] + 2L^2 \gamma^3 K_2^2 M_G^2.
	\end{aligned}
	\end{equation}
	On the other hand, for (\ref{term: cross_2}) we have
	\begin{align}
	& \frac{L\gamma^2}{2}  \EE \big\|  \frac{1}{PB} \sum\limits_{j=1}^P  \sum\limits_{s=1}^B \nabla F (\bw^j_{t};\xi^j_{t+1,s})\big\|_2^2  \\
	&\leq \frac{L\gamma^2}{2} \EE \big\| \frac{1}{PB} \sum\limits_{j=1}^P  \sum\limits_{s=1}^B  \big( \nabla F (\bw^j_{t};\xi^j_{t+1,s})  - \nabla F(\bw^j_t ) + \nabla F(\bw^j_t )\big) \big\|_2^2  \\
	& \leq \frac{L\gamma^2}{2}\EE \big\| \frac{1}{PB}  \sum\limits_{j=1}^P  \sum\limits_{s=1}^B  \big( \nabla F (\bw^j_{t};\xi^j_{t+1,s})  - \nabla F(\bw^j_t ) \big) \big\|_2^2 +  \frac{L\gamma^2}{2} \EE \big\| \frac{1}{P} \sum_{j=1}^P \nabla F(\bw^j_t) \big\|_2^2  \label{ineq: 1} \\ 
	& \leq \frac{L\gamma^2}{2P^2B^2} \sum\limits_{j=1}^P  \sum\limits_{s=1}^B \EE \big\|  \nabla F (\bw^j_{t};\xi^j_{t+1,s})  - \nabla F(\bw^j_t) \big\|_2^2  +  \frac{L\gamma^2}{2} \EE \big\| \frac{1}{P} \sum_{j=1}^P \nabla F(\bw^j_t) \big\|_2^2  \label{ineq: 2} \\
	& \leq \frac{L\gamma^2M }{2PB} + \frac{L\gamma^2}{2} \EE \big\| \frac{1}{P} \sum_{j=1}^P \nabla F(\bw^j_t) \big\|_2^2 \label{bound: cross_2}
	\end{align}
	where (\ref{ineq: 1}) is due to $P^{-1}\EE \sum_{j=1}^P  \big( \nabla F (\bw^j_{t};\xi^j_{t+1,s})  - \nabla F(\bw^j_t )\big) =0$ condition on $\bw_t^j$ and the independence over $s$. (\ref{ineq: 2}) is due to the same trick and conditional independence over $j$ and $s$.
	
	Combine (\ref{bound: cross_1}) and (\ref{bound: cross_2}), we get 
	\begin{align*}
	\EE \Big[F(\bar{\bw}_{t+1}) - F(\bar{\bw}_{t}) \Big] \leq  -\frac{\gamma}{2} \EE  \| \nabla F(\bar{\bw}_t)\|_2^2 
	- \frac{\gamma(1-L\gamma)}{2}\big\| \frac{1}{P} \sum_{j=1}^P \nabla F(\bw_t^j)  \big\|_2^2  + 2 L^2 \gamma^3 K_2^2 M_G^2  + \frac{L\gamma^2M }{2PB} .
	\end{align*}
	Take the summation over $t$, under the assumption $0< L\gamma\leq 1$ we get 
	\begin{equation}
	\frac{\gamma}{2} \sum_{t=0}^{T-1} \EE \big\| \nabla F(\bar{\bw}_t) \big\|_2^2 \leq  \Big[F(\bar{\bw}_{0}) - \EE F(\bar{\bw}_{T}) \Big]  + 2L^2  \gamma^3 K_2^2 M_G^2  + \frac{L\gamma^2M }{2PB},
	\end{equation}
	which leads to 
	\begin{equation}
	\begin{aligned}
	\frac{1}{T} \sum_{t=0}^{T-1} \EE \big\| \nabla F(\bar{\bw}_t) \big\|_2^2 &\leq \frac{2(F(\bar{\bw}_{0}) - \EE F(\bar{\bw}_{T}))}{\gamma T} + 4L^2  \gamma^2 K_2^2 M_G^2 + \frac{L\gamma M}{PB} \\
	& \leq \frac{2(F(\bar{\bw}_{0}) - F^*}{\gamma T} + 4L^2  \gamma^2 K_2^2 M_G^2 + \frac{L\gamma M}{PB}
	\end{aligned}
	\end{equation}
	By setting $\gamma = \sqrt{PB/T}$ and $K_2 = T^{1/4}/(PB)^{3/4}$, we get 
	\begin{equation}
	\frac{1}{T} \sum_{t=0}^{T-1} \EE \big\| \nabla F(\bar{\bw}_t) \big\|_2^2 \leq  \frac{2(F(\bar{\bw}_{0}) - F^*)}{\sqrt{PBT}}  +  \frac{4L^2 M_G^2}{\sqrt{PBT}} + \frac{LM}{\sqrt{PBT}}.
	\end{equation}
	
\end{proof}

\subsection{Proof of Theorem \ref{theorem: fixed}}
\label{proof: theorem1}
\begin{proof}
	We denote $\tilde{w}_{n}$ as the $n$-th global update in \Hier, denote $\bar{\bw}^j_{n+kK_1+t}$ as $t$-th local update on learner $j$ after $k$ times local averaging.
	By the algorithm, 
	$$
	\widetilde{\bw}_{n+1} - \widetilde{\bw}_n =- \frac{\gamma}{PB} \sum\limits_{j=1}^P \sum\limits_{k=0}^{\beta-1}\sum\limits_{t=0}^{K_1-1}\sum\limits_{s=1}^B \nabla F(\bar{\bw}_{n+kK_1+t}^j; \xi^j_{kK_1+t,s}).
	$$
	By the definition of SGD, the random variables $\xi^j_{kK_1+t,s}$ are i.i.d. for all $t=0,...,K_1-1$, $s=1,...,B$, $j=1,...,P$ and $k=0,...,\beta-1$.
	
	Consider 
	\begin{align}
	\EE \Big[F(\widetilde{\bw}_{n+1}) - F(\widetilde{\bw}_{n}) \Big] & \leq \EE \Big\langle \nabla F(\widetilde{\bw}_n), \widetilde{\bw}_{n+1} - \widetilde{\bw}_n \Big\rangle + \frac{L}{2} \EE \big\| \widetilde{\bw}_{n+1} - \widetilde{\bw}_n \big\|_2^2 \\
	&\leq -\gamma \Big\langle \nabla F(\widetilde{\bw}_n), \EE \sum\limits_{k=0}^{\beta-1}\sum\limits_{t=0}^{K_1-1} \nabla F(\bar{\bw}^j_{n+kK_1+t})\Big\rangle  \label{cross_term}  \\
	&+ \frac{L\gamma^2}{2P^2B^2}  \EE \big\|  \sum\limits_{j=1}^P \sum\limits_{k=0}^{\beta-1}\sum\limits_{t=0}^{K_1-1} \sum\limits_{s=1}^B \nabla F (\bar{\bw}^j_{n+kK_1+t};\xi^j_{kK_1+t,s})\big\|_2^2.\label{variance_term}
	\end{align}
	Note that here we abused the expectation notation $\EE$ a little bit. Throughout this proof, $\EE$ always means taking the overall expectation. 
	For each fixed $k$ and $t$, the random variables $\nabla F(\bar{\bw}_{n+kK_1+t}^j; \xi^j_{kK_1+t,s})$ are i.i.d. for over all $j$ and $s$ conditioning on previous steps.
	As a result, we can drop the summation over $s$ and $j$ in (\ref{cross_term}) due to the averaging factors $B$ and $P$ in the dominator.
	To be more specific,
	under the unbiasness Assumption 3, by taking the overall expectation we can immediately get 
	$$
	\EE \frac{1}{B} \sum\limits_{s=1}^B\nabla F(\bar{\bw}_{\alpha+t}^j;\xi_{\alpha+t,s}^j ) =	\EE \Big[ \frac{1}{B} \sum\limits_{s=1}^B \EE_{\xi^j_{t,s}}\nabla F(\bar{\bw}_{\alpha+t}^j;\xi_{t,s}^j | \bar{\bw}_{\alpha+t}^j)\Big]
	= \EE\nabla F(\bar{\bw}_{\alpha+t}^j).
	$$
	for fixed $j$ and $t$.
	Next, we show how to get rid of the summation over $j$. Recall that $\bar{\bw}^j_{\alpha+1} = \bar{\bw}_{\alpha} - \frac{\gamma}{B} \sum_{s=1}^B\nabla F(\bar{\bw}_{\alpha};\xi^j_{0,s})$. Obviously, $\bar{\bw}_{\alpha+1}^j$, $j=1,...,P$ are i.i.d. condiitoning on $\bar{\bw}_{\alpha}$ because $\xi_{0,s}^j$, $j=1,...,P$, $s=1,...,B$ are i.i.d. Similarly,  $\bar{\bw}^j_{\alpha+2} = \bar{\bw}^j_{\alpha+1} - \frac{\gamma}{B} \sum_{s=1}^B\nabla F(\bar{\bw}^j_{\alpha+1};\xi^j_{1,s})$, $j=1,...,P$ are i.i.d. due to the fact that 
	$\bar{\bw}_{\alpha+t}^j$'s  are i.i.d., $\xi_{1,s}^j$'s are i.i.d., and $\bar{\bw}_{\alpha+t}^j$'s are independent from $\xi_{1,s}^j$'s. By induction, one can easily show that for each fixed $t$, $\bar{\bw}_{\alpha+t}^j$, $j=1,...,P$ are i.i.d. Thus for each fixed $t$
	$$
	\frac{1}{P} \sum\limits_{j=1}^P \EE \nabla F(\bar{\bw}_{\alpha+t}^j) = \EE \nabla F(\bar{\bw}_{\alpha+t}^j).
	$$
	We can therefore get rid of the summation over $j$ as well. We will frequently use the above iterative conditional expectation trick in the following analysis.
	
	Next, we will bound (\ref{cross_term}) and (\ref{variance_term}) respectively.
	For (\ref{variance_term}), we have
	\begin{align*}
	& \frac{L\gamma^2K_1\beta}{2P^2B^2} \EE \sum\limits_{k=0}^{\beta-1}\sum\limits_{t=0}^{K_1-1} \big\|  \sum\limits_{j=1}^P \sum\limits_{s=1}^B\nabla F (\bar{\bw}^j_{n+kK_1+t};\xi^j_{kK_1+t,s})\big\|_2^2 \\
	& = \frac{L\gamma^2K_1\beta}{2P^2B^2} \EE \sum\limits_{k=0}^{\beta-1}\sum\limits_{t=0}^{K_1-1} \big\|  \sum\limits_{j=1}^P \sum\limits_{s=1}^B \big(\nabla F (\bar{\bw}^j_{n+kK_1+t};\xi^j_{kK_1+t,s})
	- \nabla F(\bar{\bw}^j_{n+kK_1+t}) + \nabla F(\bar{\bw}^j_{n+kK_1+t})\big)\big\|_2^2  \\
	& =   \frac{L\gamma^2K_1\beta}{2P^2B^2} \EE \sum\limits_{k=0}^{\beta-1}\sum\limits_{t=0}^{K_1-1} \big\|  \sum\limits_{j=1}^P \sum\limits_{s=1}^B \big(\nabla F (\bar{\bw}^j_{n+kK_1+t};\xi^j_{kK_1+t,s}) - \nabla F(\bar{\bw}^j_{n+kK_1+t}) \big)\big\|_2^2  \\
	&+  \frac{L\gamma^2K_1\beta}{2P^2B^2} \EE \sum\limits_{k=0}^{\beta-1}\sum\limits_{t=0}^{K_1-1} \big\|  \sum\limits_{j=1}^P \sum\limits_{s=1}^B \nabla F(\bar{\bw}^j_{n+kK_1+t}) \big\|_2^2  \\
	& + \frac{L\gamma^2K_1\beta}{P^2B^2}  \sum\limits_{k=0}^{\beta-1}\sum\limits_{t=0}^{K_1-1} \EE \Big\langle \sum\limits_{j=1}^P \sum\limits_{s=1}^B \big(\nabla F (\bar{\bw}^j_{n+kK_1+t};\xi^j_{kK_1+t,s}) - \nabla F(\bar{\bw}^j_{n+kK_1+t}) \big) ,   \sum\limits_{j=1}^P \sum\limits_{s=1}^B \nabla F(\bar{\bw}^j_{n+kK_1+t}) \Big\rangle  \\
	& =   \frac{L\gamma^2K_1\beta}{2P^2B^2} \EE \sum\limits_{k=0}^{\beta-1}\sum\limits_{t=0}^{K_1-1} \big\|  \sum\limits_{j=1}^P \sum\limits_{s=1}^B \big(\nabla F (\bar{\bw}^j_{n+kK_1+t};\xi^j_{kK_1+t,s}) - \nabla F(\bar{\bw}^j_{n+kK_1+t}) \big)\big\|_2^2  \\
	&+  \frac{L\gamma^2K_1\beta}{2P^2B^2} \EE \sum\limits_{k=0}^{\beta-1}\sum\limits_{t=0}^{K_1-1} \big\|  \sum\limits_{j=1}^P \sum\limits_{s=1}^B \nabla F(\bar{\bw}^j_{n+kK_1+t}) \big\|_2^2,
	\end{align*}
	where in the last equity, we used the fact that for fixed $t$ and $k$ and conditioning on $\nabla F(\bar{\bw}^j_{n+kK_1+t})$, $ \sum\limits_{j=1}^P \sum\limits_{s=1}^B \EE \big(\nabla F (\bar{\bw}^j_{n+kK_1+t};\xi^j_{kK_1+t,s}) - \nabla F(\bar{\bw}^j_{n+kK_1+t}) \big) = 0$ under unbiasness Assumption 3. Further, under the bounded variance Assumption 4, we have
	\begin{align*}
	 & \frac{L\gamma^2K_1\beta}{2P^2B^2} \EE \sum\limits_{k=0}^{\beta-1}\sum\limits_{t=0}^{K_1-1} \big\|  \sum\limits_{j=1}^P \sum\limits_{s=1}^B\nabla F (\bar{\bw}^j_{n+kK_1+t};\xi^j_{kK_1+t,s})\big\|_2^2 \\
	&  = \frac{L\gamma^2K_1\beta}{2P^2B^2} \EE \sum\limits_{k=0}^{\beta-1}\sum\limits_{t=0}^{K_1-1}  \sum\limits_{j=1}^P \sum\limits_{s=1}^B \EE \big\| \nabla F (\bar{\bw}^j_{n+kK_1+t};\xi^j_{kK_1+t,s}) - \nabla F(\bar{\bw}^j_{n+kK_1+t}) \big\|_2^2  \\
	&  + \frac{L\gamma^2K_1\beta}{2} \sum\limits_{k=0}^{\beta-1}\sum\limits_{t=0}^{K_1-1} \EE \big\| \nabla F(\bar{\bw}^j_{n+kK_1+t}) \big\|_2^2 \\
	& \leq   \frac{L\gamma^2K_1^2\beta^2 M}{2PB} + \frac{L\gamma^2K_1\beta}{2} \sum\limits_{k=0}^{\beta-1}\sum\limits_{t=0}^{K_1-1} \EE \big\| \nabla F(\bar{\bw}^j_{n+kK_1+t}) \big\|_2^2.
	\end{align*}
	Thus, we get 
	\begin{align*}
	&  \frac{L\gamma^2}{2P^2B^2}  \EE \big\|  \sum\limits_{j=1}^P \sum\limits_{k=0}^{\beta-1}\sum\limits_{t=0}^{K_1-1} \sum\limits_{s=1}^B \nabla F (\bar{\bw}^j_{n+kK_1+t};\xi^j_{kK_1+t,s})\big\|_2^2 \\
	& \leq   \frac{L\gamma^2K_1^2\beta^2 M}{2PB} + \frac{L\gamma^2K_1\beta}{2} \sum\limits_{k=0}^{\beta-1}\sum\limits_{t=0}^{K_1-1} \EE \big\| \nabla F(\bar{\bw}^j_{n+kK_1+t}) \big\|_2^2.
	\end{align*}
	Note that in the first equity we can change the summation over $j$ and $s$ out of the squared norms without introducing an extra $PB$ factor is due to the fact that conditioning on $\bar{\bw}^j_{n+kK_1+t}$, 
	$\nabla F (\bar{\bw}^j_{n+kK_1+t};\xi^j_{kK_1+t,s})$ are all independent with respect to different $j$ and $s$. In the following, we will use this trick over and over again without further explaination.
	
	For (\ref{cross_term}), we have 
	\begin{equation}
	\begin{aligned}
	\label{bias}
	&-\gamma \Big\langle F(\widetilde{\bw}_n), \EE \sum\limits_{k=0}^{\beta-1}\sum\limits_{t=0}^{K_1-1} \nabla F(\bar{\bw}^j_{n+kK_1+t})\Big\rangle  \\
	& = - \frac{\gamma}{2} \sum\limits_{k=0}^{\beta-1}\sum\limits_{t=0}^{K_1-1} \Big( \EE \big\| \nabla F(\widetilde{\bw}_n)\big\|_2^2+ \EE \big\| \nabla F(\bar{\bw}_{n+kK_1+t}^j ) \big\|_2^2 \Big)
	+ \frac{\gamma}{2} \sum\limits_{k=0}^{\beta-1}\sum\limits_{t=0}^{K_1-1} \EE \big\| \nabla F(\bar{\bw}_{n+kK_1+t}^j ) -\nabla F(\widetilde{\bw}_n )\big\|_2^2 \\
	& \leq  - \frac{\gamma \beta K_1}{2} \EE \big\| \nabla F(\widetilde{\bw}_n)\big\|_2^2 - \frac{\gamma}{2} \sum\limits_{k=0}^{\beta-1}\sum\limits_{t=0}^{K_1-1} \EE \big\| \nabla F (\bar{\bw}^j_{n+kK_1+t})\big\|_2^2
	+ \frac{\gamma L^2}{2}   \sum\limits_{k=0}^{\beta-1}\sum\limits_{t=0}^{K_1-1} \EE \big\| \bar{\bw}_{n+kK_1+t}^j - \widetilde{\bw}_n  \big\|_2^2 ,
	\end{aligned}
	\end{equation}
	where we used the Lipschitz Assumption 1 in the last inequality.
	
	In the following lemma, we derive a general bound on $ \EE \big\| \bar{\bw}_{n+kK_1+t}^j - \widetilde{\bw}_n  \big\|_2^2$.
	
	\begin{lemma}
		\label{lemma: 1}
		For any $t\in \{0,1,2,...,K_1-1\}$ and $\eta \in \{0,1,2,...,\beta-1\}$, we have 
		\begin{equation}
			 \EE \big\| \bar{\bw}_{n+kK_1+t}^j - \widetilde{\bw}_n  \big\|_2^2 \leq  \frac{\gamma^2M}{B}  (K_1\eta +t)\Big(t+ \frac{K_1\eta}{S} \Big) + \gamma^2 (K_1\eta+t) \sum\limits_{k=0}^{K_1\eta+t}  \EE \Big\| \nabla F(\bar{\bw}_{n+k}^j) \Big\|_2^2
		\end{equation}
	\end{lemma}

   \begin{proof}
   		Recall that for any $P_j$ in a local cluster $P_{lc}$ with $|P_{lc}| = S$,
   	\begin{align*}
   	&\bar{\bw}_{n+kK_1+t}^j - \widetilde{\bw}_n  \\
   	&=  \frac{\gamma}{BS} \sum\limits_{j\in P_{lc}}  \sum\limits_{\eta=0}^{k-1} \sum\limits_{r=0}^{K_1-1}\sum\limits_{s=1}^B \nabla F(\bar{\bw}_{n+\eta K_1+r}^j; \xi_{\eta K_1+r,s}^j) + \frac{\gamma}{B} \sum\limits_{i=0}^{t-1}\sum\limits_{s=1}^B \nabla F(\bar{\bw}_{n+k K_1+i}^j; \xi_{\eta K_1+i,s}^j).
   	\end{align*}
   	Therefore,
   	\begin{align}
   	&\EE \big\| \bar{\bw}_{n+\eta K_1+ t}^j - \widetilde{\bw}_n \big\|_2^2 \\
   	& =  \EE \Big\| \sum\limits_{i=0}^{t-1} \frac{\gamma}{B} \sum\limits_{s=1}^B \nabla F(\bar{\bw}_{n+\eta K_1+i}^j;\xi_{n+\eta K_1+i,s}^j)  
   	+ \frac{\gamma}{BS} \sum\limits_{j\in P_{lc} } \sum\limits_{k=0}^{\eta-1}  \sum\limits_{r=0}^{K_1-1} \sum\limits_{s=1}^B  \nabla F(\bar{\bw}_{n+k K_1+r}^j;\xi_{n+k K_1+r,s}^j) \Big\|_2^2 \\
   	& \leq  \frac{\gamma^2}{B^2}  (K_1\eta+t)  \sum\limits_{i=0}^{t-1}  \EE \big\| \sum\limits_{s=1}^B \nabla F(\bar{\bw}_{n+\eta K_1+i}^j;\xi_{n+ \eta K_1+i,s}^j) \big\|_2^2 \label{vterm_1} \\
   	& +  \frac{\gamma^2}{B^2S^2}  (K_1\eta +t)  \sum\limits_{k=0}^{\eta-1} \sum\limits_{r=0}^{K_1-1} \EE  \big\| \sum\limits_{j\in P_{lc}} \sum\limits_{s=1}^B \nabla F(\bar{\bw}_{n+k K_1+r}^j;\xi_{k K_1+r,s}^j) \big\|_2^2 \label{vterm_2}
   	\end{align}
   	For term (\ref{vterm_1})
   	\begin{align}
   	& \frac{\gamma^2}{B^2}  (K_1\eta+t)  \sum\limits_{i=0}^{t-1}  \EE \big\| \sum\limits_{s=1}^B \nabla F(\bar{\bw}_{n+\eta K_1+i}^j;\xi_{n+ \eta K_1+i,s}^j) \big\|_2^2 \\
   	& =  \frac{\gamma^2}{B^2}  (K_1\eta+t)  \sum\limits_{i=0}^{t-1}  \EE \big\| \sum\limits_{s=1}^B \Big( \nabla F(\bar{\bw}_{n+\eta K_1+i}^j;;\xi_{n+ \eta K_1+i,s}^j) - \nabla F(\bar{\bw}_{n+\eta K_1+i}^j) +\nabla F(\bar{\bw}_{n+\eta K_1+i}^j)\Big)\big\|_2^2 \\
   	& \leq  \frac{\gamma^2}{B^2}  (K_1\eta+t)  \sum\limits_{i=0}^{t-1}  \sum\limits_{s=1}^B \EE \Big\| \nabla F(\bar{\bw}_{n+\eta K_1+i}^j;;\xi_{n+ \eta K_1+i,s}^j) - \nabla F(\bar{\bw}_{n+\eta K_1+i}^j) \Big\|_2^2 \\
   	&+ \gamma^2 (K_1\eta+t)  \sum\limits_{i=0}^{t-1}  \EE \Big\| \nabla F(\bar{\bw}_{n+\eta K_1+i}^j) \Big\|_2^2 \\
   	& \leq  \frac{\gamma^2M}{B}  (K_1\eta +t)t +\gamma^2 (K_1\eta+t)  \sum\limits_{i=0}^{t-1}  \EE \Big\| \nabla F(\bar{\bw}_{n+\eta K_1+i}^j) \Big\|_2^2 \label{vterm_3}.
   	\end{align}
   	Similarly, for term (\ref{vterm_2}) we have
   	\begin{align}
   	&\frac{\gamma^2}{B^2S^2}  (K_1\eta +t)  \sum\limits_{k=0}^{\eta-1} \sum\limits_{r=0}^{K_1-1} \EE  \big\| \sum\limits_{j\in P_{lc}} \sum\limits_{s=1}^B \nabla F(\bar{\bw}_{n+k K_1+r}^j;\xi_{k K_1+r,s}^j) \big\|_2^2 \\
   	& \leq  \frac{\gamma^2M}{BS}  (K_1\eta +t)K_1\eta +\gamma^2 (K_1\eta+t)  \sum\limits_{k=0}^{\eta-1} \sum\limits_{r=0}^{K_1-1}  \EE \Big\| \nabla F(\bar{\bw}_{n+k K_1+r}^j) \Big\|_2^2 \label{vterm_4}.
   	\end{align}
   	Combine (\ref{vterm_3}) and (\ref{vterm_4}), we get 
   	\begin{align*}
   	&\EE \big\| \bar{\bw}_{n+\eta K_1+ t}^j - \widetilde{\bw}_n \big\|_2^2 
   	 \leq \frac{\gamma^2M}{B}  (K_1\eta +t)\Big(t+ \frac{K_1\eta}{S} \Big) + \gamma^2 (K_1\eta+t) \sum\limits_{k=0}^{K_1\eta+t-1}  \EE \Big\| \nabla F(\bar{\bw}_{n+k}^j) \Big\|_2^2.
   	\end{align*}
   \end{proof}

	Therefore, using the result of Lemma \ref{lemma: 1}
	\begin{align*}
		&\frac{\gamma L^2}{2}   \sum\limits_{\eta =0}^{\beta-1}\sum\limits_{t=0}^{K_1-1} \EE \big\| \bar{\bw}_{n+\eta K_1+t}^j - \widetilde{\bw}_n  \big\|_2^2 \\
		& \leq  \frac{ L^2 \gamma^3 M}{2B} \sum\limits_{\eta =0}^{\beta-1}\sum\limits_{t=0}^{K_1-1} (K_1\eta +t)\Big(t+ \frac{K_1\eta}{S} \Big) + 
		\frac{L^2\gamma^3M}{2} \sum\limits_{\eta =0}^{\beta-1}\sum\limits_{t=0}^{K_1-1}  (K_1\eta+t) \sum\limits_{k=0}^{K_1\eta+t-1}  \EE \Big\| \nabla F(\bar{\bw}_{n+k}^j) \Big\|_2^2 \\
		& = \frac{ L^2 \gamma^3 MK_2}{24B} \Big( \frac{(K_2-K_1)(4K_2 + K_1 -3)}{S} + (K_1-1)(3K_2+K_1 -2)\Big) \\
		& + \frac{L^2\gamma^3MK_2(K_2-1) }{2}  \EE \Big\| \nabla F(\tilde{\bw}_n)  \Big\|_2^2  \\
		& + \frac{L^2\gamma^3M}{2} \sum\limits_{\eta =0}^{\beta-1}\sum\limits_{t=0}^{K_1-1}  (K_1\eta+t) \mathbf{1}\{K_1\eta+ t -1 \geq 1\} \sum\limits_{k=1}^{K_1\eta+t-1}  \EE \Big\| \nabla F(\bar{\bw}_{n+k}^j) \Big\|_2^2 .
	\end{align*}
	
	Then we will have an upper bound on $ \sum\limits_{\eta =0}^{\beta-1}\sum\limits_{t=0}^{K_1-1}  (K_1\eta+t) \mathbf{1}\{K_1\eta+ t -1 \geq 1\} \sum\limits_{k=1}^{K_1\eta+t-1}  \EE \Big\| \nabla F(\bar{\bw}_{n+k}^j) \Big\|_2^2$.
	
	\begin{lemma}
		\label{Lemma: 2}
	   \begin{equation}
	   \begin{aligned}
	    &\sum\limits_{\eta =0}^{\beta-1}\sum\limits_{t=0}^{K_1-1}  (K_1\eta+t) \mathbf{1}\{K_1\eta+t -1 \geq 1\} \sum\limits_{k=1}^{K_1\eta+t-1}  \EE \Big\| \nabla F(\bar{\bw}_{n+k}^j) \Big\|_2^2 \\
	    & \leq \Big(\frac{K_2(K_2-1)}{2}-1 -\delta_{\nabla F,\bw} \Big)\sum\limits_{k=1}^{K_2-1} \EE \Big\| \nabla F(\bar{\bw}^j_{n+k}) \Big\|_2^2,
	    \end{aligned}
	   \end{equation}
	   where $\delta \in (0, K_2(K_2-3)/2)$ is a constant depending on the immediate gradient norms  $\| \nabla F(\bar{\bw}^j_{n+k}) \|_2^2$, $k=1,...,K_2-1$.
	\end{lemma}
	\begin{proof}
	Obviously, $\EE \Big\| \nabla F(\bar{\bw}_{n+1}^j) \Big\|_2^2 $ has the most copies, we will derive an upper bound 
	on the number of $\EE \Big\| \nabla F(\bar{\bw}_{n+1}^j) \Big\|_2^2 $ and then use this bound to uniformly bound the number of terms for $\EE \Big\| \nabla F(\bar{\bw}_{n+k}^j) \Big\|_2^2 $, $k=1,...,K_2-2$.
	\begin{align*}
		&\sum\limits_{\eta =0}^{\beta-1}\sum\limits_{t=0}^{K_1-1}  (K_1\eta+t) \mathbf{1}\{K_1\eta+ t -1 \geq 1\}  \\
		& \leq \sum\limits_{t=0}^{K_1-1}  t \mathbf{1}\{t  \geq 2\}\mathbf{1}\{K_1  \geq 3\} + \sum\limits_{t=0}^{K_1-1}  (K_1 + t) \mathbf{1}\{K_1 + t  \geq 2\}\mathbf{1}\{K_1  \geq 2\} \\
		& + \sum\limits_{\eta =2}^{\beta-1}\sum\limits_{t=0}^{K_1-1}  (K_1\eta+t) \mathbf{1}\{K_1\eta+ t  \geq 2\} \\
		& = \frac{(K_1-2)(K_1+1)}{2}\mathbf{1}\{K_1  \geq 3\} + \frac{K_1(3K_1-1)}{2}\mathbf{1}\{K_1  \geq 2\}
		+ \frac{(K_2-2K_1)(K_2+2K_1-1)}{2}\\
		&\leq \frac{K_2(K_2-1)}{2}-1.
		\end{align*}
	\end{proof}
	
	Following Lemma \ref{Lemma: 2}, we get 
	\begin{equation}
	\label{combine123}
		\begin{aligned}
		&\frac{\gamma L^2}{2}   \sum\limits_{\eta =0}^{\beta-1}\sum\limits_{t=0}^{K_1-1} \EE \big\| \bar{\bw}_{n+\eta K_1+t}^j - \widetilde{\bw}_n  \big\|_2^2 \\
		& < \frac{ L^2 \gamma^3 MK_2}{24B} \Big( \frac{(K_2-K_1)(4K_2 + K_1 -3)}{S} + (K_1-1)(3K_2+K_1 -2)\Big) \\
		& + \frac{L^2\gamma^3MK_2(K_2-1) }{2}  \EE \Big\| \nabla F(\tilde{\bw}_n)  \Big\|_2^2  \\
		& + \frac{L^2\gamma^3}{2}\Big(\frac{K_2(K_2-1)}{2}-1 -\delta_{\nabla F,\bw} \Big)\sum\limits_{k=1}^{K_2-1} \EE \Big\| \nabla F(\bar{\bw}^j_{n+k}) \Big\|_2^2.
		\end{aligned}
		\end{equation}
		
	Plug (\ref{combine123}) into (\ref{bias}), we have
	\begin{equation}
	\begin{aligned}
	\label{bias_final}
	& -\gamma \Big\langle F(\widetilde{\bw}_n), \EE \sum\limits_{k=0}^{\beta-1}\sum\limits_{t=0}^{K_1-1} \nabla F(\bar{\bw}^j_{n+kK_1+t})\Big\rangle \\
	&  \leq - \frac{\gamma (K_2+1)}{2}\Big[1 - \frac{L^2\gamma^2 K_2(K_2-1)}{2(K_2+1)}\Big] \EE \big\| \nabla F(\widetilde{\bw}_n)\big\|_2^2 \\
	&  - \frac{\gamma}{2} \Big( 1- L^2\gamma^2\Big(\frac{K_2(K_2-1)}{2}-1 -\delta_{\nabla F,\bw} \Big)\sum\limits_{k=1}^{K_2-1} \EE \Big\| \nabla F(\bar{\bw}^j_{n+k}) \Big\|_2^2 \\
	& + \frac{ L^2 \gamma^3 MK_2}{24B} \Big( \frac{(K_2-K_1)(4K_2 + K_1 -3)}{S} + (K_1-1)(3K_2+K_1 -2)\Big) 
	\end{aligned}
	\end{equation}
	
	Plug (\ref{variance_term}) and (\ref{bias_final}) back into $\EE\Big[F(\widetilde{\bw}_{n+1}) - F(\widetilde{\bw}_{n}) \Big]$, we get
	\begin{equation}
	\label{equation: 1}
	\begin{aligned}
	& \EE\Big[F(\widetilde{\bw}_{n+1}) - F(\widetilde{\bw}_{n}) \Big] \\
	& \leq - \frac{\gamma (K_2+1)}{2}\Big[1 - \frac{L^2\gamma^2 (K_2-1)K_2}{2(K_2+1)} - \frac{L\gamma K_2}{K_2+1}\Big] \EE \big\| \nabla F(\widetilde{\bw}_n)\big\|_2^2 \\
	&  - \frac{\gamma}{2} \Big( 1-  L^2\gamma^2\Big(\frac{K_2(K_2-1)}{2}-1 -\delta_{\nabla F,\bw} \Big)- L\gamma K_2   \Big)\sum\limits_{k=1}^{K_2-1} \EE \Big\| \nabla F(\bar{\bw}^j_{n+k}) \Big\|_2^2 \\
	& +  \frac{ L^2 \gamma^3 MK_2}{24B} \Big( \frac{(K_2-K_1)(4K_2 + K_1 -3)}{S} + (K_1-1)(3K_2+K_1 -2)\Big)   +    \frac{L\gamma^2 K_2^2 M}{2PB} 
	\end{aligned}
	\end{equation}
	
	Under the condition,
	$$
	 1-  L^2\gamma^2\Big(\frac{K_2(K_2-1)}{2}-1 -\delta_{\nabla F,\bw} \Big)- L\gamma K_2  \geq 0,
	$$
	we have
	$$
	\frac{\gamma (K_2+1)}{2}\Big[1 - \frac{L^2\gamma^2 (K_2-1)K_2}{2(K_2+1)} - \frac{L\gamma K_2}{K_2+1}\Big] \geq  \frac{\gamma}{2}\Big(K_2 - L^2\gamma^2 (1+\delta_{\nabla F, \bw})\Big).
	$$	
	We can therefore drop the second term on the right hand side in (\ref{equation: 1}) and take the summation over $n$ to get 
	\begin{equation}
	\begin{aligned}
	& \EE\Big[F(\widetilde{\bw}_{N}) - F(\widetilde{\bw}_{1}) \Big] \\
	& \leq -\frac{\gamma}{2}\Big(K_2 - L^2\gamma^2 (1+\delta_{\nabla F, \bw})\Big) \sum\limits_{n=1}^N \EE \big\| \nabla F(\widetilde{\bw}_n)\big\|_2^2 \\
	& +  \frac{ L^2 \gamma^3 MK_2}{24B} \Big( \frac{(K_2-K_1)(4K_2 + K_1 -3)}{S} + (K_1-1)(3K_2+K_1 -2)\Big)   +    \frac{L\gamma^2 K_2^2 M}{2PB}
	\end{aligned}
	\end{equation}
	Under Assumption 2, we have 
	\begin{equation}
		F^* -  F(\widetilde{\bw}_{1}) \leq F(\widetilde{\bw}_{N}) - F(\widetilde{\bw}_{1}).
	\end{equation}
	As a result,
	\begin{align*}
	&   \frac{\gamma}{2}\Big(K_2 - L^2\gamma^2 (1+\delta_{\nabla F, \bw})\Big) \sum\limits_{n=1}^N \EE \big\| \nabla F(\widetilde{\bw}_n)\big\|_2^2\\
	& \leq \EE\Big[F(\widetilde{\bw}_{1}) - F^* \Big]  +   \frac{ L^2 \gamma^3 MK_2}{24B} \Big( \frac{(K_2-K_1)(4K_2 + K_1 -3)}{S} + (K_1-1)(3K_2+K_1 -2)\Big)   +    \frac{L\gamma^2 K_2^2 M}{2PB}
	\end{align*}
	Thus we have
	\begin{align*}
	& \frac{1}{N} \sum\limits_{n=1}^N \EE \big\| \nabla F(\widetilde{\bw}_n)\big\|_2^2  \leq   \frac{2\EE[F(\widetilde{\bw}_{1}) - F^*]}{N[K_2 - L^2\gamma^2 (1+\delta_{\nabla F, \bw})]\gamma}   
	+ \frac{L\gamma M K_2^2 }{PB [K_2 - L^2\gamma^2 (1+\delta_{\nabla F, \bw})]}  \\ 
	& +  \frac{L^2\gamma^2 M K_2}{12B[K_2 - L^2\gamma^2 (1+\delta_{\nabla F, \bw})]} \Big( \frac{(K_2-K_1)(4K_2 + K_1 -3)}{S} + (K_1-1)(3K_2+K_1 -2)\Big).
	\end{align*}
\end{proof}

\subsection{Proof of Theorem \ref{theorem: diminishing}}
\label{proof: theorem2}
\begin{proof}
	The proof of Theorem \ref{theorem: diminishing} is similar to that of Theorem \ref{theorem: fixed}. Indeed, under the condition (\ref{condition: gamma}),
	\begin{equation*}
	\label{proof_condition: gamma_j}
      1-  L^2\gamma^2_j\Big(\frac{K_2(K_2-1)}{2}-1 -\delta_{\nabla F,\bw} \Big)- L\gamma_j K_2  \geq 0,
	\end{equation*}
	we have
	\begin{equation}
	\label{equation: 2}
	\frac{\gamma_j (K_2+1)}{2}\Big[1 - \frac{L^2\gamma_j^2 K_2(K_2-1)}{2(K_2+1)} - \frac{L\gamma_j K_2}{K_2+1}\Big] \geq \frac{\gamma_j}{2}\Big(
	K_2 - L^2\gamma_j^2 (1+\delta_{\nabla F,\bw})\Big).
	\end{equation}
	Meanwhile, from (\ref{proof_condition: gamma_j}), we have $L^2\gamma_j^2(1+\delta_{\nabla F,\bw}) \leq L^2\gamma_j^2K_2^2/2 \leq 1$, thus $K_2 - L^2\gamma_j^2 (1+\delta_{\nabla F,\bw}) \geq K_2-1$. 
	By replacing $\gamma$ with $\gamma_j$ in (\ref{equation: 1}) together with (\ref{equation: 2}), we have 
	\begin{equation}
	\begin{aligned}
      &\frac{\gamma_j(K_2-1)}{2}\EE \big\| \nabla F(\widetilde{\bw}_j)\big\|_2^2   \leq \EE\Big[F(\widetilde{\bw}_{j+1}) - F(\widetilde{\bw}_{j}) \Big] \\
      &+  \frac{L\gamma_j^2 K_2^2 M}{2PB} +  \frac{L^2\gamma_j^3 MK_2}{24B}\Big( \frac{(K_2-K_1)(4K_2 + K_1 -3)}{S} + (K_1-1)(3K_2+K_1 -2)\Big)
      \end{aligned}
	\end{equation}
	Taking the summation over $j$, and divide both sides by $\sum_{j=1}^N \gamma_j$, we got
	\begin{equation*}
	\begin{aligned}
	&  \EE \sum\limits_{j=1}^N \frac{\gamma_j}{\sum_{j=1}^N \gamma_j} \big\| \nabla F(\widetilde{\bw}_j)\big\|_2^2  \leq   \frac{2\EE[F(\widetilde{\bw}_{1}) - F^*]}{(K_2-1)\sum_{j=1}^N\gamma_j}   
	+ \sum\limits_{j=1}^N\frac{L M K_2^2 \gamma_j^2}{PB_j (K_2-1)\sum_{j=1}^N \gamma_j }  \\ 
	& + \sum\limits_{j=1}^N \frac{L^2 MK_2 \gamma^3_j}{12B_j(K_2-1)\sum_{j=1}^N\gamma_j}\Big( \frac{(K_2-K_1)(4K_2 + K_1 -3)}{S} + (K_1-1)(3K_2+K_1 -2)\Big) .
     \end{aligned}
     \end{equation*}
	\end{proof}

\subsection{Proof of Theorem \ref{theorem: K_2}}
\label{proof: theoremK2}
\begin{proof}
	Under the assumption $T=N*K_2$, we can rewrite the bound (\ref{bound:fixed}) as
	\begin{align*}
	& \frac{1}{N} \sum\limits_{n=1}^N \EE \big\| \nabla F(\widetilde{\bw}_n)\big\|_2^2  \leq   \frac{2\EE[F(\widetilde{\bw}_{1}) - F^*]K_2}{T(K_2 - \delta)\gamma}   
	+ \frac{L\gamma M K_2^2 }{PB (K_2 - \delta)}  \\ 
	& +  \frac{L^2\gamma^2 M K_2}{12B(K_2 - \delta)} \Big( \frac{(K_2-K_1)(4K_2 + K_1 -3)}{S} + (K_1-1)(3K_2+K_1 -2)\Big)
	\end{align*}
	To move on, we set 
	$$
	B(K_2) := f(K_2) * g(K_2)
	$$
	where 
	$$
	f(K_2) := \Big(\alpha+ \beta K_2+ \eta \Big( \frac{(K_2-K_1)(4K_2 + K_1 -3)}{S} + (K_1-1)(3K_2+K_1 -2)\Big)\Big)
	$$
	and 
	$$
	g(K_2) :=  \Big(\frac{K_2}{K_2- \delta}\Big),~\alpha = \frac{2[\EE F(\widetilde{\bw}_1)-F^*]}{T\gamma},~\beta = \frac{L\gamma M}{PB},~\eta= \frac{L^2\gamma^2M}{12B}.
	$$
	
	To minimize the right hand side of (\ref{bound:fixed}), it is equivalent to solve the following integer program
	$$
	K_2^* = \min\limits_{K_2 \in \NN^*} B(K_2),
	$$
	which can be very hard. Meanwhile, one should notice that $K_2^*$ depends on some unknown quantities such as $L$, $M$ and $(F(\widetilde{\bw}_1)-F^*)$. 
	Instead, we investigate the monotonicity of $B(K_2)$. 
	Firstly, we show that $f(K_2)$ is non-decreasing.
	\begin{lemma}
		Given $K_2 \geq K_1 \geq 1$, $f(K_2)$ is non-decreasing.
	\end{lemma}
	
	\begin{proof}
		The key is to show that $(K_2-K_1)(4K_2 + K_1 -3)/S $ is non-decreasing with respect to $K_2$. It is easy to see that the quadratic function $(K_2-K_1)(4K_2 + K_1 -3)/S $ is non-decreasing with respect to $K_2$
		when $K_2 \geq 3(K_1+1)/8$, which is always true given $K_2\geq K_1 \geq 1$. Thus, $(K_2-K_1)(4K_2 + K_1 -3)/S $ is monotone increasing, so is $f(K_2)$.
	\end{proof}

	On the other hand, $g(K_2)$ is monotone decreasing for $K_2 \geq 1$. Therefore, $B(K_2)$ is a multiplication of an increasing function and a decreasing one.
	Thus, a sufficient condition for $K_2^*>1$ is that $B(2)<B(1)$, which is equivalent to
	$$
	\frac{\delta\alpha}{1-\delta} > 2\beta + \frac{12\eta}{S}.
	$$
\end{proof}

\subsection{Proof of Theorem \ref{theorem: K_1}}
\label{proof: theoremK1}
\begin{proof}
	The proof of part 2 is obvious, so we omit it here. For part 1,
	With $K_2$ fixed, it is sufficient to consider the monotonicity of $(K_2-K_1)(4K_2 + K_1 -3)/S + (K_1-1)(3K_2+K_1 -2) $ for both bounds in (\ref{bound:fixed}) and (\ref{bound:diminishing}). Set
	\begin{equation}
	f(K_1) =  \frac{(K_2-K_1)(4K_2 + K_1 -3)}{S} + (K_1-1)(3K_2+K_1 -2).
	\end{equation}
	Then 
	$$
	f'(K_1) = \frac{(S-1)(3K_2+2K_1 -3)}{S}.
	$$
	Apparently, $f(K_1)$ is monotone increasing with respect to $K_1$ when $K_1 \geq 2$ given $S>1$ and $K_2 \geq K_1$. 
\end{proof}

\subsection{Proof of Theorem \ref{theorem: comparison}}
\label{proof: theorem_comp}
\begin{proof}
	\label{proof: theorem_comparison}
	When $P$ is large enough such that $L\gamma P \gg 1$, the second term in bound (\ref{bound:fixed}) is dominanted by the third term.
	We omit the second term in (\ref{bound:fixed}) and denote it by $\mathcal{H}(K)$ and get 
	$$
	\mathcal{H}(K) := f_1(K) * g_1(K)
	$$
	where 
	$$
	f_1(K) := \Big(\alpha+  \eta \Big( \frac{ \big((1+a)K-1\big)   \big(2(1+a)K-1\big) }{4} \Big)
	$$
	and 
	$$
	g_1(K) :=  \Big(\frac{(1+a)K}{(1+a)K- \delta}\Big),~\alpha = \frac{2[\EE F(\widetilde{\bw}_1)-F^*]}{T\gamma},~\eta= \frac{L^2\gamma^2M}{6B}.
	$$
	On the other hand, we denote the similar bound of \KAVG as $\chi(K)$ by plugging in $K_2=K$, $K_1=1$, $S=1$ in (\ref{theorem: fixed}) (also see \cite{fan2018kavg}), which is 
	\begin{equation}
	\chi(K) :=  f_2(K)*g_2(K),
	\end{equation}
	where 
	$$
	f_2(K) := \alpha  + \eta (K-1)(2K-1)  ,~~g_2(K) :=  \Big(\frac{K}{K- \delta}\Big).
	$$
	In the following, we show that $\mathcal{H}(K)< \chi(K)$ uniformly for all $K\geq 2$ and $a\in [0,0.443]$. It is easy to see that $0<g_1(K) < g_2(K)$ for all $K\geq 2$.
	Next, we show that $f_1(K) < f_2(K)$ for all $K\geq 2$ and $a\in [0, 0.443]$. Set 
	\begin{equation}
		F(K) :=f_1(K) - f_2(K).
	\end{equation}
	Then $F(K)$ is a quadratic function of $K$, 
	\begin{equation}
		F'(K) = ((1+a)^2-4)K + 3 -\frac{3(1+a)}{4},
	\end{equation}
	Then it is easy to see that for all $a\in [0,0.779]$ and $F'(K)\leq 0$ for all $K\geq 2$. Meanwhile, 
	\begin{equation}
	 F(2) = 2(1+a)^2 - \frac{3}{2}(1+a) -\frac{11}{4},
	\end{equation}
	and it is easy to check that $F(2) <0$ for all $a\in [0,0.606]$.
	As a result, $F(K) < 0$ for all $K\geq 2$ and $a\in[0,0.606]$. This implies that $\mathcal{H}(K) < \chi(K)$ for all $K\geq 2$ and all $a\in [0,0.606]$.

\end{proof}

\bibliographystyle{plainnat}
\bibliography{refer}

\begin{thebibliography}{26}
\providecommand{\natexlab}[1]{#1}
\providecommand{\url}[1]{\texttt{#1}}
\expandafter\ifx\csname urlstyle\endcsname\relax
  \providecommand{\doi}[1]{doi: #1}\else
  \providecommand{\doi}{doi: \begingroup \urlstyle{rm}\Url}\fi

\bibitem[Bottou et~al.(2018)Bottou, Curtis, and
  Nocedal]{bottou2018optimization}
L{\'e}on Bottou, Frank~E Curtis, and Jorge Nocedal.
\newblock Optimization methods for large-scale machine learning.
\newblock \emph{SIAM Review}, 60\penalty0 (2):\penalty0 223--311, 2018.

\bibitem[Chen et~al.(2016)Chen, Pan, Monga, Bengio, and
  Jozefowicz]{chen2016revisiting}
Jianmin Chen, Xinghao Pan, Rajat Monga, Samy Bengio, and Rafal Jozefowicz.
\newblock Revisiting distributed synchronous sgd.
\newblock \emph{arXiv preprint arXiv:1604.00981}, 2016.

\bibitem[Dean et~al.(2012)Dean, Corrado, Monga, Chen, Devin, Mao, Senior,
  Tucker, Yang, Le, et~al.]{dean2012large}
Jeffrey Dean, Greg Corrado, Rajat Monga, Kai Chen, Matthieu Devin, Mark Mao,
  Andrew Senior, Paul Tucker, Ke~Yang, Quoc~V Le, et~al.
\newblock Large scale distributed deep networks.
\newblock In \emph{Advances in neural information processing systems}, pages
  1223--1231, 2012.

\bibitem[Dekel et~al.(2012)Dekel, Gilad-Bachrach, Shamir, and
  Xiao]{dekel2012optimal}
Ofer Dekel, Ran Gilad-Bachrach, Ohad Shamir, and Lin Xiao.
\newblock Optimal distributed online prediction using mini-batches.
\newblock \emph{Journal of Machine Learning Research}, 13\penalty0
  (Jan):\penalty0 165--202, 2012.

\bibitem[Deng et~al.(2009)Deng, Dong, Socher, Li, Li, and
  Fei-Fei]{imagenet_cvpr09}
J.~Deng, W.~Dong, R.~Socher, L.-J. Li, K.~Li, and L.~Fei-Fei.
\newblock {ImageNet: A Large-Scale Hierarchical Image Database}.
\newblock In \emph{CVPR09}, 2009.

\bibitem[Ghadimi and Lan(2013)]{ghadimi2013stochastic}
Saeed Ghadimi and Guanghui Lan.
\newblock Stochastic first-and zeroth-order methods for nonconvex stochastic
  programming.
\newblock \emph{SIAM Journal on Optimization}, 23\penalty0 (4):\penalty0
  2341--2368, 2013.

\bibitem[Hazan and Kale(2014)]{hazan2014beyond}
Elad Hazan and Satyen Kale.
\newblock Beyond the regret minimization barrier: optimal algorithms for
  stochastic strongly-convex optimization.
\newblock \emph{The Journal of Machine Learning Research}, 15\penalty0
  (1):\penalty0 2489--2512, 2014.

\bibitem[He et~al.(2016)He, Zhang, Ren, and Sun]{he2016deep}
Kaiming He, Xiangyu Zhang, Shaoqing Ren, and Jian Sun.
\newblock Deep residual learning for image recognition.
\newblock In \emph{Proceedings of the IEEE conference on computer vision and
  pattern recognition}, pages 770--778, 2016.

\bibitem[Howard et~al.(2017)Howard, Zhu, Chen, Kalenichenko, Wang, Weyand,
  Andreetto, and Adam]{howard2017mobilenets}
Andrew~G Howard, Menglong Zhu, Bo~Chen, Dmitry Kalenichenko, Weijun Wang,
  Tobias Weyand, Marco Andreetto, and Hartwig Adam.
\newblock Mobilenets: Efficient convolutional neural networks for mobile vision
  applications.
\newblock \emph{arXiv preprint arXiv:1704.04861}, 2017.

\bibitem[Johnson and Zhang(2013)]{johnson2013accelerating}
Rie Johnson and Tong Zhang.
\newblock Accelerating stochastic gradient descent using predictive variance
  reduction.
\newblock In \emph{Advances in neural information processing systems}, pages
  315--323, 2013.

\bibitem[Krizhevsky and Hinton(2009)]{krizhevsky2009learning}
Alex Krizhevsky and Geoffrey Hinton.
\newblock Learning multiple layers of features from tiny images.
\newblock 2009.

\bibitem[Li et~al.(2014)Li, Andersen, Park, Smola, Ahmed, Josifovski, Long,
  Shekita, and Su]{li2014scaling}
Mu~Li, David~G Andersen, Jun~Woo Park, Alexander~J Smola, Amr Ahmed, Vanja
  Josifovski, James Long, Eugene~J Shekita, and Bor-Yiing Su.
\newblock Scaling distributed machine learning with the parameter server.
\newblock In \emph{OSDI}, volume~1, page~3, 2014.

\bibitem[Lin et~al.(2018)Lin, Stich, and Jaggi]{lin2018don}
Tao Lin, Sebastian~U Stich, and Martin Jaggi.
\newblock Don't use large mini-batches, use local sgd.
\newblock \emph{arXiv preprint arXiv:1808.07217}, 2018.

\bibitem[Loshchilov and Hutter(2016)]{loshchilov2016sgdr}
Ilya Loshchilov and Frank Hutter.
\newblock Sgdr: stochastic gradient descent with restarts.
\newblock \emph{Learning}, 10:\penalty0 3, 2016.

\bibitem[Recht et~al.(2011)Recht, Re, Wright, and Niu]{recht2011hogwild}
Benjamin Recht, Christopher Re, Stephen Wright, and Feng Niu.
\newblock Hogwild: A lock-free approach to parallelizing stochastic gradient
  descent.
\newblock In \emph{Advances in neural information processing systems}, pages
  693--701, 2011.

\bibitem[Robbins and Monro(1951)]{robbins1951stochastic}
Herbert Robbins and Sutton Monro.
\newblock A stochastic approximation method.
\newblock \emph{The annals of mathematical statistics}, pages 400--407, 1951.

\bibitem[Simonyan and Zisserman(2014)]{simonyan2014very}
Karen Simonyan and Andrew Zisserman.
\newblock Very deep convolutional networks for large-scale image recognition.
\newblock \emph{arXiv preprint arXiv:1409.1556}, 2014.

\bibitem[Smith et~al.(2016)Smith, Forte, Ma, Tak{\'a}c, Jordan, and
  Jaggi]{smith2016cocoa}
Virginia Smith, Simone Forte, Chenxin Ma, Martin Tak{\'a}c, Michael~I Jordan,
  and Martin Jaggi.
\newblock Cocoa: A general framework for communication-efficient distributed
  optimization.
\newblock \emph{arXiv preprint arXiv:1611.02189}, 2016.

\bibitem[Stich(2018)]{stich2018local}
Sebastian~U Stich.
\newblock Local sgd converges fast and communicates little.
\newblock \emph{arXiv preprint arXiv:1805.09767}, 2018.

\bibitem[Szegedy et~al.(2015)Szegedy, Liu, Jia, Sermanet, Reed, Anguelov,
  Erhan, Vanhoucke, and Rabinovich]{szegedy2015going}
Christian Szegedy, Wei Liu, Yangqing Jia, Pierre Sermanet, Scott Reed, Dragomir
  Anguelov, Dumitru Erhan, Vincent Vanhoucke, and Andrew Rabinovich.
\newblock Going deeper with convolutions.
\newblock In \emph{Proceedings of the IEEE conference on computer vision and
  pattern recognition}, pages 1--9, 2015.

\bibitem[Wang et~al.(2017)Wang, Wang, and Srebro]{wang2017memory}
Jialei Wang, Weiran Wang, and Nathan Srebro.
\newblock Memory and communication efficient distributed stochastic
  optimization with minibatch prox.
\newblock \emph{arXiv preprint arXiv:1702.06269}, 2017.

\bibitem[Wang and Joshi(2018)]{wang2018adaptive}
Jianyu Wang and Gauri Joshi.
\newblock Adaptive communication strategies to achieve the best error-runtime
  trade-off in local-update sgd.
\newblock \emph{arXiv preprint arXiv:1810.08313}, 2018.

\bibitem[Yu et~al.(2018)Yu, Yang, and Zhu]{yu2018parallel}
Hao Yu, Sen Yang, and Shenghuo Zhu.
\newblock Parallel restarted sgd for non-convex optimization with faster
  convergence and less communication.
\newblock \emph{arXiv preprint arXiv:1807.06629}, 2018.

\bibitem[Zhang et~al.(2016)Zhang, De~Sa, Mitliagkas, and
  R{\'e}]{zhang2016parallel}
Jian Zhang, Christopher De~Sa, Ioannis Mitliagkas, and Christopher R{\'e}.
\newblock Parallel sgd: When does averaging help?
\newblock \emph{arXiv preprint arXiv:1606.07365}, 2016.

\bibitem[Zhou and Cong(2018)]{fan2018kavg}
Fan Zhou and Guojing Cong.
\newblock On the convergence properties of a k-step averaging stochastic
  gradient descent algorithm for nonconvex optimization.
\newblock In \emph{Proceedings of the Twenty-Seventh International Joint
  Conference on Artificial Intelligence, {IJCAI-18}}, pages 3219--3227, 2018.

\bibitem[Zinkevich et~al.(2010)Zinkevich, Weimer, Li, and
  Smola]{zinkevich2010parallelized}
Martin Zinkevich, Markus Weimer, Lihong Li, and Alex~J Smola.
\newblock Parallelized stochastic gradient descent.
\newblock In \emph{Advances in neural information processing systems}, pages
  2595--2603, 2010.

\end{thebibliography}

\end{document}